\documentclass{article}
\usepackage{iclr2027_conference,times}

\usepackage{amsmath,amsfonts,bm}

\def\eqref#1{equation~\ref{#1}}

\def\1{\bm{1}}

\DeclareMathAlphabet{\mathsfit}{\encodingdefault}{\sfdefault}{m}{sl}
\SetMathAlphabet{\mathsfit}{bold}{\encodingdefault}{\sfdefault}{bx}{n}

\newcommand{\E}{\mathbb{E}}

\newcommand{\R}{\mathbb{R}}

\newcommand{\KL}{D_{\mathrm{KL}}}

\DeclareMathOperator*{\argmax}{arg\,max}

\usepackage{amsmath,amssymb,amsthm,mathtools}
\usepackage{xcolor}
\usepackage{enumitem}
\usepackage{graphicx}
\usepackage{booktabs}
\usepackage{arydshln}
\usepackage{algorithm}
\usepackage[noend]{algpseudocode}
\usepackage{float}
\usepackage{wrapfig}
\usepackage{capt-of}
\usepackage{placeins}
\usepackage{hyperref}
\usepackage{fontawesome5}
\usepackage{url}
\usepackage[capitalize,nameinlink]{cleveref}
\definecolor{scoredark}{HTML}{8B0000}  %
\definecolor{scorebg}{HTML}{8B0000}
\definecolor{darkgreen}{rgb}{0,0.5,0}
\definecolor{darkred}{rgb}{0.7,0,0}
\definecolor{teal}{rgb}{0.1,0.6,0.7}
\definecolor{blue}{rgb}{0.0,0.1,0.9}
\definecolor{orange}{rgb}{1.,0.7,0.0}
\definecolor{palegreen}{rgb}{0.7,0.7,0.0}
\definecolor{lightblue}{rgb}{0.70, 0.80, 0.89}
\definecolor{violet}{rgb}{0.50, 0.16, 0.88}
\definecolor{babyblue}{rgb}{0.00, 0.88, 0.88}
\definecolor{electricpurple}{rgb}{0.75, 0.0, 1.0}

\newcount\Comments  %
\newcommand{\kibitz}[2]{\ifnum\Comments=1{{\textcolor{#1}{\textsf{\footnotesize [#2]}}}}\fi}

\usepackage{aliascnt}
\newcommand{\newaliastheorem}[3]{%
  \newaliascnt{#1}{theorem}%
  \newtheorem{#1}[#1]{#2}%
  \aliascntresetthe{#1}%
  \crefname{#1}{#2}{#3}%
  \Crefname{#1}{#2}{#3}%
}

\theoremstyle{plain}

\crefname{theorem}{Theorem}{Theorems}
\Crefname{theorem}{Theorem}{Theorems}
\crefname{algorithm}{Algorithm}{Algorithms}
\Crefname{algorithm}{Algorithm}{Algorithms}
\newaliastheorem{lemma}{Lemma}{Lemmas}
\newaliastheorem{proposition}{Proposition}{Propositions}
\newaliastheorem{corollary}{Corollary}{Corollaries}
\newaliastheorem{counterexample}{Counterexample}{Counterexamples}
\theoremstyle{definition}
\newaliastheorem{definition}{Definition}{Definitions}
\newaliastheorem{assumption}{Assumption}{Assumptions}
\theoremstyle{remark}
\newaliastheorem{remark}{Remark}{Remarks}

\usepackage[most]{tcolorbox}
\newtcolorbox{takeaway}[1]{
    enhanced,
    breakable,
    colback=black!3!yellow!8,
    colframe=black,
    boxrule=0.8pt,
    arc=2.5pt,
    outer arc=2.5pt,
    left=7pt,
    right=7pt,
    top=8pt,
    bottom=6pt,
    before skip=4pt,
    after skip=4pt,
    fonttitle=\small\bfseries,
    coltitle=white,
    title={#1},
    attach boxed title to top left={
        xshift=10pt,
        yshift=-6pt
    },
    boxed title style={
        colback=black,
        colframe=black,
        boxrule=0pt,
        arc=2pt,
        outer arc=2pt,
        left=4pt,
        right=4pt,
        top=2pt,
        bottom=1pt,
    },
    fontupper=\small,
}

\newcommand{\TV}{\mathrm{TV}}

\newcommand{\diag}{\operatorname{diag}}
\newcommand{\PhiMap}{\Phi}                    %

\newcommand{\Vocab}{\mathcal{V}}              %
\newcommand{\eos}{\texttt{[EOS]}}
\newcommand{\Bdom}{\mathcal{B}}               %

\newcommand{\norm}[1]{\left\lVert #1 \right\rVert}
\newcommand{\Reach}{\mathcal{R}}
\newcommand{\diam}{\operatorname{diam}}
\newcommand{\tr}{\operatorname{tr}}

\newcommand{\epsint}{\epsilon_{\mathrm{int}}}   %
\newcommand{\epscov}{\epsilon_{\mathrm{cov}}}   %

\crefformat{section}{\S#2#1#3}
\crefformat{subsection}{\S#2#1#3}
\crefformat{subsubsection}{\S#2#1#3}
\crefmultiformat{section}{\S#2#1#3}{ and~\S#2#1#3}{, #2#1#3}{ and~#2#1#3}

\title{

BOReFT: Manifold Steering of\\ Language Models for Black-box Optimization

}

\author{Dhruv Agarwal$^{\ast}$, Rico Angell$^{\dagger}$, Kavitha Srinivas$^{\ddagger}$, Tahira Naseem$^{\ddagger}$, \\
\textbf{Horst Samulowitz$^{\ddagger}$, Willie Neiswanger$^{\S}$, Andrew McCallum$^{\ast}$} \\[4pt]
$^{\ast}$University of Massachusetts Amherst, $^{\dagger}$New York University, \\
$^{\ddagger}$IBM Research, $^{\S}$University of Southern California \\
\texttt{\{dagarwal,mccallum\}@cs.umass.edu}\\\\
\href{https://github.com/dhdhagar/boreft}{\raisebox{-0.1ex}{\normalfont\faGithub}\,\texttt{github.com/dhdhagar/boreft}}
}

\iclrfinalcopy[preprint]

\begin{document}

\maketitle

\begin{abstract}
Language models are increasingly used as proposal models for black-box search, from program optimization to molecular design.
Existing approaches typically improve proposals through iterative prompting or parameter updates, offering limited control over how completely and efficiently the model's search space is explored.
Continuous optimization methods, such as Bayesian optimization, provide a principled way to search but require a suitable domain to operate over.
To address this, we introduce BOReFT, which learns a compact, low-dimensional 
space of hidden-state interventions in a frozen language model, and uses this space as the search domain for Bayesian optimization with an external scoring function.
Empirically, we find that the learned domain 
spans semantic regions and exhibits smoothness properties that support search.
Theoretically, we show that semantic coverage and interpolation control the best score available in the learned space, and that 
decoding from this space yields a standard stochastic-bandit observation model for adaptive search.
We evaluate BOReFT on the interpretable word search task ``Semantle'' and on 
three more real-world discovery tasks in \emph{de novo} molecule property optimization.
Compared to strong LLM baselines, BOReFT finds in Semantle a higher 
number of hidden targets
and, on two out of three molecular objectives, achieves higher property scores. 
Consequently, our method provides 
a principled new 
bridge between discrete proposal spaces of LLM-based search and continuous black-box optimization.
\end{abstract}

\section{Introduction}
\label{sec:introduction}
\begin{wrapfigure}{r}{0.53\textwidth}
    \vspace{-2.1em}
    \centering
    \includegraphics[width=\linewidth]{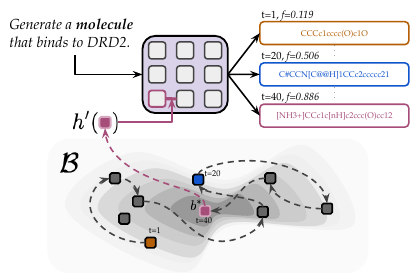}
    \vspace{-5mm}
    \caption{\textbf{BOReFT overview.}
    Bayesian optimization searches a learned intervention manifold $\Bdom$, where each code $b$ steers an LLM activation $h'$ to generate a discrete candidate scored by objective $f$.
    }
    \label{fig:overview}
    \vspace{-1em}
\end{wrapfigure}
Generative search and discovery using language models proposes new discrete solutions, such as molecules or scientific hypotheses, under a limited budget of external evaluations.
In several settings, the evaluation yields only scalar feedback; thus, efficiently finding a high-scoring proposal becomes a black-box optimization problem.
Bayesian optimization provides a principled way to perform such a search, but operates over a continuous decision domain, whereas a language model generates discrete, stochastic sequences.

Existing approaches bridge this mismatch indirectly, either through prompting or by searching representations of candidate solutions \citep{yang2024opro,agarwal2025bopro,chen2024instructzero,kristiadi2024sober,rankovic2025gollum}.
They do not search a continuous space that can be decoded directly into solutions.

BOReFT learns such a search domain inside the frozen language model.
From target sequences paired with descriptions, it learns a compact, low-dimensional manifold of hidden-state interventions.
Each code in the learned manifold specifies an intervention that the frozen decoder maps to a distribution over discrete sequences, and the learned codes define a bounded, continuous domain that Bayesian optimization searches over using only external scores (\cref{fig:overview}).
This separates learning a useful generative search domain from the downstream procedure used to explore it.

\paragraph{Contributions.} We make three contributions: \textbf{(i)} we introduce BOReFT, which learns a compact intervention manifold in a frozen language model and searches that fixed space with Bayesian optimization (\cref{sec:method}); \textbf{(ii)} we decompose performance into a representational gap determined by the learned space and a search gap determined at inference, and characterize both through semantic coverage, interpolation, and standard stochastic-bandit observations (\cref{sec:theory-prob}); \textbf{(iii)} we evaluate BOReFT on Semantle and molecular property optimization, and analyze how representation-training data, loss components, coverage, and interpolation affect downstream search (\cref{sec:experiments}).

\section{Related work}
\label{sec:related}
\paragraph{BO with language models.}
Classical latent-space BO searches continuous latent variables of a generative model and decodes them into structured objects \citep{gomez2018automatic,maus2022local}.
\citet{kristiadi2024sober} use frozen or parameter-efficiently finetuned LLMs as feature extractors for molecular BO, and GOLLuM learns an LLM deep kernel jointly with a Gaussian process \citep{rankovic2025gollum}.
BOPRO fits a surrogate over external embeddings and proposes by prompting \citep{agarwal2025bopro}, and GGOLLuM trains a LoRA actor from preferences supplied by a GP critic \citep{rankovic2026ggollum}.
Notably, InstructZero applies BO to a low-dimensional soft-prompt that a frozen LLM consumes \citep{chen2024instructzero}, but has subsequently been shown to be ineffective in practice \citep{agarwal2025bopro}.
BOReFT learns the continuous coordinates from target solutions and searches them inside the frozen generator.

\paragraph{Activation steering and representation interventions.}
Mechanistic-interpretability methods recover structure in activations, including sparse autoencoder features \citep{cunningham2024sparse,templeton2024scaling}, and steering methods intervene on those activations to change generation \citep{subramani2022steering,rimsky2024steering,zou2023repe}.
ReFT and LoReFT learn low-rank interventions for task adaptation \citep{wu2024reft}.
BOReFT uses that intervention form as the decision variables of black-box search.
\Cref{app:related} discusses test-time search and related interpretability methods.

\section{BOReFT: Bayesian optimization via Representation Finetuning}
\label{sec:method}
BOReFT is a two-phase method for learning and searching a continuous intervention space in a frozen language model.
During training, it learns a low-dimensional space from target sequences paired with auxiliary descriptions.
At inference, the learned domain is fixed, and Bayesian optimization searches over it to optimize an external black-box objective.
We view the learned space of hidden-state interventions as an \emph{intervention manifold}, whose coordinates provide a continuous search interface to the language model.

\subsection{Low-rank interventions as a search space}
\label{sec:method-intervention}

Rather than searching over prompts or model parameters, BOReFT finds continuous interventions on an internal representation built from LoReFT \citep{wu2024reft,geiger2024das}.
Let $M$ be a frozen transformer language model with hidden dimension $d$, and let $h\in\R^d$ be the residual stream activation at a chosen layer and prompt position.
For intervention rank $r\ll d$, LoReFT maps $h$ to
\begin{equation}
\label{eq:method-intervention}
h'
\;=\;
h + R^\top(Wh+b-Rh),
\qquad
R,W\in\R^{r\times d},
\quad
b\in\R^r.
\end{equation}
The rows of $R$ are the subspace in which the intervention acts, while the base model is kept frozen.
In BOReFT, we share $R$ and $W$ across a set of targets and treat the bias $b$ as a target-dependent code.
For fixed $h$, \cref{eq:method-intervention} is then an affine transformation in that subspace as a function of $b$,
\begin{equation}
\label{eq:method-affine}
h'(b)=h_0+R^\top b,
\qquad
h_0:=h+R^\top(Wh-Rh),
\end{equation}
The intervention is applied once, at the chosen prompt position, and the remaining layers decode from $h'(b)$ as usual.
We write $P_b^\tau$ for the output distribution over sequences induced by $h'(b)$ at temperature $\tau$ (or $P_b$, when $\tau=1$).
The layer and position for each experiment are in \cref{sec:experiments}.

\subsection{Learning the search space}
\label{sec:method-train}

We train on pairs $\{(s_i,d_i)\}_{i=1}^n$, where $s_i$ is a target sequence and $d_i$ is an auxiliary natural-language description of the same target.
For tasks in this work, on word search, $s_i$ is a word and $d_i$ is its dictionary-like natural language definition,
while for molecular search, $s_i$ is a SMILES string and $d_i$ is its natural-language description from ChEBI-20 \citep{edwards2021text2mol}.

\paragraph{Posterior over codes.}
Distribution-wise interventions cover a neighborhood in representation space and control behavior more reliably than pointwise interventions \citep{deng2025distributionwise}.
We therefore use a VAE-style parameterization \citep{kingma2014vae}, where each target--description pair is first passed through a frozen \textbf{semantic encoder}\footnote{Depending on the experiment, $\phi$ is either an external embedding model or a representation extracted from the frozen language model itself; the exact choice is given in \cref{sec:experiments}.} $\phi$, then through a shared trainable projection $g_\psi$ that maps it to the parameters of a diagonal Gaussian over intervention codes,
\begin{equation}
\label{eq:method-posterior}
q_i(b)
\;=
\mathcal{N}\!\bigl(\mu_i,\,\diag(\sigma_i^2)\bigr),
\qquad
(\mu_i,\sigma_i)=g_\psi\!\bigl(\phi(s_i,d_i)\bigr),
\qquad
\mu_i,\sigma_i\in\R^r.
\end{equation}
Notice that since $g_\psi$ predicts all posterior parameters, the target codes are not learned as unrelated free vectors but inherit structure from the semantic representations of $\phi$.
We use this as a geometric inductive bias, although we do not assume that semantic distances are preserved exactly.
During training we sample
$b_i=\mu_i+\sigma_i\odot\epsilon$ with $\epsilon\sim\mathcal{N}(0,I_r)$ using the reparameterization trick.

\paragraph{Trainable parameters.} The base model $M$ and semantic encoder $\phi$ remain frozen, while the intervention matrices $R,W$ and the projection parameters $\psi$ are set as trainable parameters.
For each target, $\mu_i$ and $\sigma_i$ are produced by the shared map $g_\psi(\phi(s_i,d_i))$, and the training code $b_i$ is sampled from the resulting posterior $q_i$.

\paragraph{Training objective.}
For effective search in the continuous intervention space, we would like the observed targets to be recoverable from their codes, the search space to remain compact, and codes to retain useful related behavior rather than act dictionary lookups.
We encourage these properties by constructing a loss objective that
combines reconstruction, variational prior regularization, and on-policy self-distillation, i.e., for target $i$,
\begin{equation}
\label{eq:method-loss}
\mathcal{L}^{(i)}
\;=\;
\mathcal{L}_{\mathrm{rec}}^{(i)}
\;+\;
\beta\,\mathcal{L}_{\mathrm{prior}}^{(i)}
\;+\;
\lambda\,\mathcal{L}_{\mathrm{distill}}^{(i)}.
\end{equation}
Here, reconstruction anchors each training target in the learned space, 
the variational prior penalizes posteriors that move far from a shared reference distribution, 
and self-distillation helps learn a richer distribution of output sequences beyond the target itself.
In \cref{sec:theory-prob-training}, we analyze how the three terms help shape the space that search later uses.

\paragraph{Reconstruction.}
Let $s_i=(s_{i,1},\ldots,s_{i,T_i})$ be the tokens of target $i$.
For a code $b$ sampled from $q_i$, we use teacher-forced next-token cross-entropy to reproduce that target,
\begin{equation}
\label{eq:method-rec}
\mathcal{L}_{\mathrm{rec}}^{(i)}
\;=\;
\E_{b\sim q_i}
\left[
-\frac{1}{T_i}
\sum_{t=1}^{T_i}
\log P_b\!\left(s_{i,t}\mid s_{i,<t}\right)
\right],
\end{equation}
where $s_{i,<t}$ is the prefix before position $t$.

\paragraph{Variational prior.}
We use an isotropic Gaussian prior $p_0=\mathcal{N}(0,\sigma_0^2 I_r)$ to keep the posteriors at a shared scale, and penalize the KL divergence between each posterior and this prior,
\begin{equation}
\label{eq:method-prior}
\mathcal{L}_{\mathrm{prior}}^{(i)}
\;=\;
\KL(q_i\,\|\,p_0).
\end{equation}
This discourages posteriors from moving arbitrarily far apart, thus limiting the spread of the learned space and keeping search efficient.

\paragraph{On-policy self-distillation.}
Sampling from proposal distributions has been shown to improve LLM search \citep{brown2024monkeys,wang2023selfconsistency}; so the learned codes should emit a distribution of outputs instead of a single sequence.
We enable this via knowledge distillation \citep{hinton2015distilling} on on-policy rollouts from the posterior\footnote{In practice, we include target sequence $s_i$ resulting in a mixed policy \citep{yan2026learning,phan2025migrate}.}, where the student is the intervened model $P_b$ and the teacher $Q_i$ is a frozen copy of the unintervened model conditioned on the description $d_i$ as privileged information \citep{zhao2026selfdistilled,hubotter2026reinforcement}.
The teacher then supplies dense token-level supervision via
\begin{equation}
\label{eq:method-distill}
\mathcal{L}_{\mathrm{distill}}^{(i)}
\;=\;
\E_{\mathcal{T}_i}
\E_{b\sim q_i}
\left[
\frac{1}{\sum_{S\in\mathcal{T}_i}|S|}
\sum_{S\in\mathcal{T}_i}
\sum_{t=1}^{|S|}
\KL\!\left(
Q_i(\cdot\mid S_{<t})
\,\middle\|\,
P_b(\cdot\mid S_{<t})
\right)
\right].
\end{equation}
Unlike in previous work, we use forward KL here to cover all behaviors exhibited by the teacher.

\paragraph{Search domain.}
After training, we no longer sample from the target posteriors to propose search points.
Instead, we construct a fixed search domain from the learned posterior means using their axis-aligned bounding box \citep{frazier2018tutorial,siivola2021good}, which is the Cartesian product of $r$ component-wise intervals,
\begin{equation}
\label{eq:method-box}
\Bdom
\;=\;
\prod_{j=1}^{r}
\left[
\min_i \mu_{i,j},\;
\max_i \mu_{i,j}
\right] \subset \mathbb{R}^r.
\end{equation}
Training therefore produces a generative map $b\mapsto P_b^\tau$ together with a bounded continuous domain $\Bdom$ over which that map can be queried, rather than just the $n$ training means.

\begin{figure}[t]
\centering
\begin{minipage}[c]{0.46\textwidth}
\begin{algorithm}[H]
\footnotesize
\caption{BOReFT: learning and searching a low-rank generative space}
\label{alg:boreft}
\begin{algorithmic}[1]
\Require Frozen LM $M$; pairs $\{(s_i,d_i)\}_{i=1}^n$
\Statex \hspace{\algorithmicindent}encoder $\phi$; objective $f$; budget $Q$
\Statex \textbf{Phase I: learn the search space}
\For{each training update}
    \State $(\mu_i,\sigma_i)\gets g_\psi(\phi(s_i,d_i))$
    \State $b_i=\mu_i+\sigma_i\odot\epsilon$,\; $\epsilon\sim\mathcal{N}(0,I_r)$
    \State Sample rollouts from $P_{b_i}$
    \Statex \hspace{\algorithmicindent}\hspace{\algorithmicindent}(optionally include $s_i$) to form $\mathcal{T}_i$
    \State Update $R,W,\psi$ on \cref{eq:method-rec,eq:method-prior,eq:method-distill}
\EndFor
\State $\Bdom\gets\prod_{j=1}^r[\min_i\mu_{ij},\max_i\mu_{ij}]$
\Statex \hspace{\algorithmicindent}from the trained posterior means
\Statex \textbf{Phase II: search the learned space}
\For{$t=1,\ldots,Q$}
    \State Fit surrogate on observed code--score pairs
    \State Optimize acquisition function
    \Statex \hspace{\algorithmicindent}\hspace{\algorithmicindent}$b_t\in\arg\max_{b\in\Bdom} a_t(b)$
    \State Decode $S_t\sim P_{b_t}^{\tau_{\mathrm{search}}}$
    \State Evaluate $y_t=f(S_t)$
\EndFor
\State \Return the highest-scoring sequence
\end{algorithmic}
\end{algorithm}
\end{minipage}\hfill
\begin{minipage}[c]{0.52\textwidth}
\centering
\includegraphics[width=\linewidth]{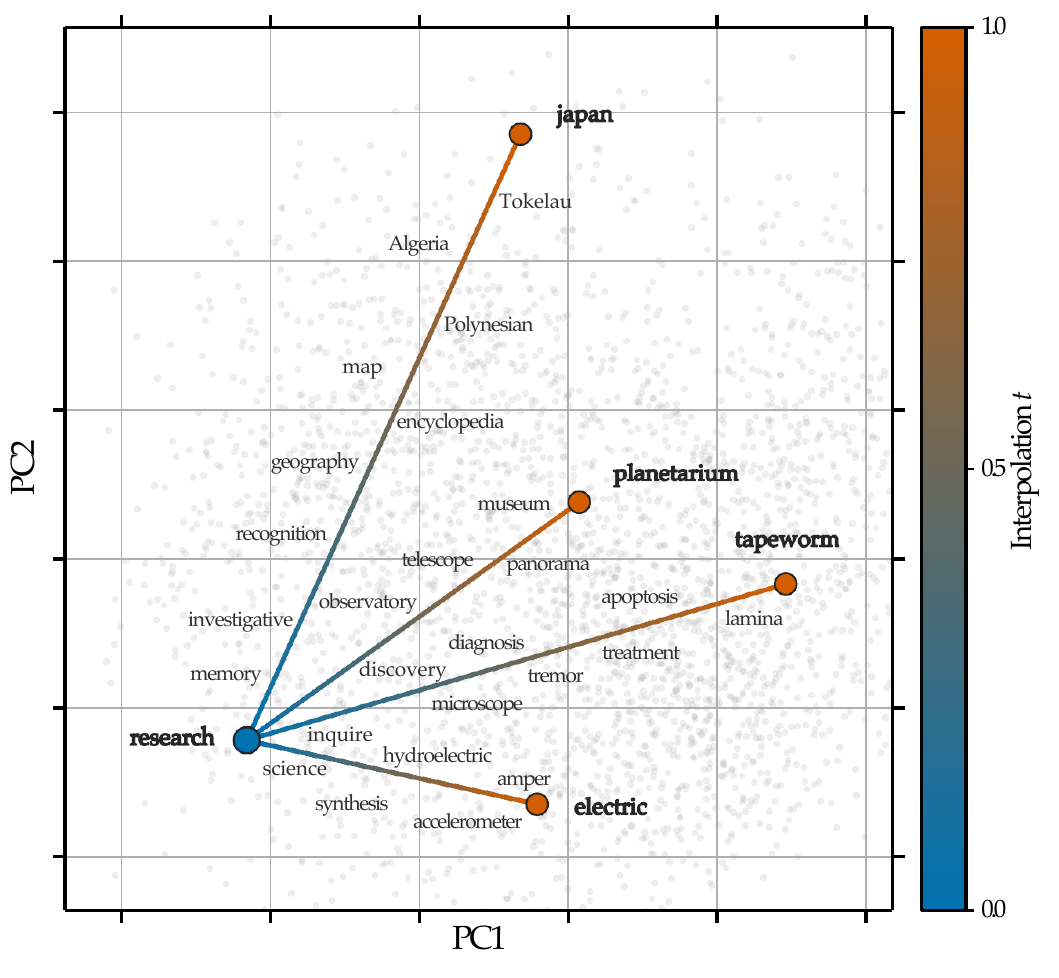}
\caption{\textbf{Interpolation in the learned code space.}
PCA of the $3072$ Semantle training posterior means.
Color denotes interpolation weight $t$.}
\label{fig:interp-code-pca}
\end{minipage}
\end{figure}

\subsection{Searching the learned space with Bayesian optimization}
\label{sec:method-search}

At inference, the domain $\Bdom$ is fixed and search varies only the code $b\in\Bdom$.
A query chooses a code, decodes at temperature $\tau_{\mathrm{search}}$, and returns a score using an external black-box objective $f$,
\begin{equation}
\label{eq:method-query}
b_t\in\Bdom
\;\longrightarrow\;
S_t\sim P_{b_t}^{\tau_{\mathrm{search}}}
\;\longrightarrow\;
y_t=f(S_t).
\end{equation}
To choose the next code, we use Bayesian optimization over $\Bdom$ \citep{shahriari2016taking,frazier2018tutorial} as a principled framework for budgeted black-box search over continuous domains.
After $t$ queries, a Gaussian process surrogate \citep{rasmussen2003gaussian} is fit to the observed pairs $\{(b_j,y_j)\}_{j=1}^t$, and the next code is selected by maximizing an acquisition function $a_t$ built from the surrogate,
\begin{equation}
\label{eq:method-acq}
b_{t+1}
\in
\argmax_{b\in\Bdom} \, a_t(b),
\end{equation}
which defines the search strategy.
In our main experiments, $a_t$ is log expected improvement \citep{ament2023unexpected}.
After $Q$ evaluations, we return the highest-scoring sequence observed during the run.
\Cref{alg:boreft} summarizes the full training and search procedure.

\section{What can search over the learned code space return?}
\label{sec:theory-prob}

At inference, BOReFT is the search procedure of \cref{sec:method} in which each query proposes a code $b \in \Bdom$, the frozen decoder generates a sequence, and the task objective scores it.
We ask how training shapes that space (\cref{sec:theory-prob-training}) and when adaptive search over it is sample-efficient (\cref{sec:theory-prob-bo}).
The usual i.i.d.\ generalization framework does not apply, because held-out discovery candidates are not drawn from the training distribution, and success is measured by finding \emph{some} code in $\Bdom$ that makes the frozen decoder emit a high-scoring sequence. 

\subsection{The search problem and a decomposition of its regret}
\label{sec:theory-prob-setup}
\label{sec:theory-prob-decomp}

Each query draws $b$ from the box $\Bdom$ of \cref{eq:method-box}, then samples $S \sim P^\tau_b$ to return $f(S)$.
For $\tau>0$, this is a noisy evaluation of the expected sampled score,
\begin{equation}
\label{eq:theory-prob-fexp}
F_\tau(b) \;:=\; \E_{S \sim P^\tau_b}\big[f(S)\big] \;=\; \sum_{s} P^\tau_b(s)\, f(s),
\end{equation}
where the sum runs over the finitely many sequences the decoder can emit under the length cap.\footnote{Under greedy decoding, the observation is the score of a deterministic sequence, which is piecewise constant in $\Bdom$. Sampling at $\tau > 0$ replaces it with a noisy observation whose expectation is continuous in $b$.} We write $F^\star_\tau := \max_{b \in \Bdom} F_\tau(b)$ for the best expected score available in the learned space, and $f^\star := \max_s f(s)$ for the best score of any sequence
attained at a sequence $s^\star$.

\paragraph{Regret decomposition.} The regret of a query at a code $b \in \Bdom$, measured against the optimum, then splits into two nonnegative parts,
\begin{equation}
\label{eq:theory-prob-regret}
f^\star - F_\tau(b)
\;=\;
\underbrace{f^\star - F^\star_\tau}_{\text{representational gap}}
\;+\;
\underbrace{F^\star_\tau - F_\tau(b)}_{\text{search gap}} .
\end{equation}
The model, rank, and training fix the representational gap, and further search cannot reduce it.
The search gap, on the other hand, is controlled by the inference procedure.
\Cref{sec:theory-prob-training} asks when $F^\star_\tau$ is large, and \cref{sec:theory-prob-bo} asks how efficiently can adaptive search over $\Bdom$ find it.

\subsection{How training shapes the space that is searched}
\label{sec:theory-prob-training}

\paragraph{Effect of the loss components.} We first describe the effect of each loss component on the learned space. While reconstruction makes each training target searchable from its code, it does not by itself guarantee that a held-out sequence is reachable in $\Bdom$ (\cref{prop:theory-prob-nfl})\footnote{In the spirit of the No Free Lunch theorems \citep{wolpert1996lack,wolpert1997no}. 
}.
Next, as discussed in \cref{sec:method-train}, the VAE prior limits how far the means spread; \cref{lem:theory-prob-train}(i) shows that it bounds the resultant search domain $\Bdom$ by the rank and the scale of the prior.
Finally, when the expected forward KL to the teacher is small, self-distillation makes the student track the teacher on every set of sequences, after averaging over the posterior (\cref{lem:theory-prob-train}(ii)).

\paragraph{Semantic interpolation and coverage.} A property that makes the learned space useful for search is that moving between learned anchors changes the decoded semantics gradually, so a code between anchors can express meaning that no single training target provides.
\Cref{fig:interp-code-pca} shows this on a 2D PCA projection of the learned Semantle space: temperature-sampled decodes along linear interpolations of the training posterior means follow different semantic directions for different anchors.

Whether this property helps a specific search task depends on that task respecting the same semantic geometry. 
We make this link explicit below by requiring the task score to vary smoothly with distance in the semantic representation. 
Without such a condition, semantic interpolation alone need not imply similar task value.
Let $\phi(s)$ be the semantic embedding from \cref{sec:method-train}, and let $\mu_i$ be the learned mean of training target $s_i$.
For convex weights $\alpha$ in the simplex $\mathcal{S}_n$, we compare the semantic interpolant of the targets with the interpolant of their codes,
\begin{equation}
\label{eq:theory-prob-interp}
x_\alpha \;:=\; \sum_i \alpha_i\, \phi(s_i),
\qquad
b_\alpha \;:=\; \sum_i \alpha_i\, \mu_i.
\end{equation}
The intervention is affine in the code (\cref{eq:method-affine}), so $h'(b_\alpha)=\sum_i \alpha_i h'(\mu_i)$, an instantiation that some prior work has also used \citep{park2024linear,subramani2022steering}.
However, the nonlinear decoder may still result in sequences that deviate from the interpolant, which must be measured empirically. %
Over a nonempty set of interpolation weights $\mathcal{A} \subseteq \mathcal{S}_n$, and for any sequence $s$, we therefore define
\begin{equation}
\label{eq:theory-prob-epsint}
\epsint \;:=\; \sup_{\alpha \in \mathcal{A}} \; \E_{S \sim P^\tau_{b_\alpha}} \norm{\phi(S) - x_\alpha}_2,
\qquad
\epscov(s) \;:=\; \inf_{\alpha \in \mathcal{A}} \; \norm{\phi(s) - x_\alpha}_2,
\end{equation}
as the \textbf{interpolation error} of the learned space and the \textbf{semantic coverage error} of $s$, respectively. The first measures how far decoded sequences lie from the semantic interpolant of the anchors in the embedding space, and the second measures how far a given sequence lies from the semantic interpolants made available by the training set.

\begin{proposition}[Semantic coverage and interpolation bound the representational gap]
\label{prop:theory-prob-interp}
Suppose the task score $f$ is Lipschitz with respect to the semantic embedding, in that $|f(s) - f(s')| \le L_f \norm{\phi(s) - \phi(s')}_2$ for all sequences $s, s'$ and some constant $L_f \ge 0$ set by the task, and let the interpolation weights $\mathcal{A} \subseteq \mathcal{S}_n$ be nonempty. Then, for any sequence $s$,
\begin{equation}
\label{eq:theory-prob-interpbound}
F^\star_\tau \;\ge\; f(s) \;-\; L_f \big( \epscov(s) + \epsint \big) .
\end{equation}
(Proof in \cref{app:prob-interp}.)
\end{proposition}

The Lipschitz condition, a property of the task relative to $\phi$, links interpolation to search, by which sequences that are close under $\phi$ have similar task scores.
Under this, taking $s = s^\star$ bounds the representational gap of \cref{eq:theory-prob-regret} by $L_f(\epscov(s^\star) + \epsint)$.
The Lipschitz condition holds for Semantle, whose score is cosine similarity between normalized embeddings, so $L_f = 1$.
For the molecular predictors it remains an assumption.
In \cref{sec:experiments} we measure coverage of held-out targets and the displacement of the mean decoded semantics along interpolated codes,
which is at most $\epsint$ (\cref{app:prob-interp}).

\begin{takeaway}{Takeaway}
Reconstruction anchors known behaviors at codes, the prior keeps them inside a compact search region,
and self-distillation distributes the teacher outputs on codes around the anchors.
If nearby embeddings have similar task scores, the representational gap is bounded by how well the optimal sequence is covered by the train set and how close the decoder stays to semantic interpolants (\cref{prop:theory-prob-interp}).
\end{takeaway}

\subsection{Guarantees for adaptive search}
\label{sec:theory-prob-bo}

After training, the remaining question is how efficiently an adaptive procedure closes the \emph{search gap} of \cref{eq:theory-prob-regret}. In each round, a new query is selected based on the task scores observed so far, returning a noisy evaluation of the expected sampled score $F_\tau$ of \cref{eq:theory-prob-fexp}.
Here we assume the task score is bounded in an interval $[f_{\min}, f_{\max}]$ and write $\Delta_f := f_{\max} - f_{\min}$ for its range.

\begin{proposition}[Queries are standard stochastic bandit observations]
\label{prop:theory-prob-bo}
Let each query $b_t \in \Bdom$ be proposed by any rule that depends on the past queries and observations $\mathcal{H}_{t-1}$, for example by maximizing an acquisition function fit to them, and let $Y_t := f(S_t)$ with $S_t \sim P^\tau_{b_t}$. Then $\E[Y_t \mid \mathcal{H}_{t-1}, b_t] = F_\tau(b_t)$, and the observation error $\xi_t := Y_t - F_\tau(b_t)$, the deviation of the sampled score from its conditional expectation, is conditionally $(\Delta_f/2)$-sub-Gaussian; i.e., its tails decay at least as fast as those of a Gaussian with standard deviation $\Delta_f/2$. (Proof in \cref{app:prob-bo}.)
\end{proposition}

In words, each query is an unbiased, conditionally sub-Gaussian observation of $F_\tau$, for any policy that proposes the code. Standard bandit and Bayesian optimization guarantees apply to this observation model only under their own further assumptions.
Discovery depends on the best sequence a run emits, which the next corollary controls for any policy.
\begin{corollary}[Cumulative regret to discovery guarantee]
\label{cor:theory-prob-bo-best}
Let $R_Q := \sum_{t \le Q} \big( F^\star_\tau - F_\tau(b_t) \big)$ be the cumulative regret of any policy against the best expected score $F^\star_\tau$ over $\Bdom$, and let $S_1, \dots, S_Q$ be the sequences its queries return. Then, with probability at least $1 - \delta$,
\begin{equation}
\label{eq:theory-prob-bobest}
\max_{t \le Q} f(S_t) \;\ge\; F^\star_\tau \;-\; \frac{R_Q}{Q} \;-\; \Delta_f \sqrt{\frac{\log(1/\delta)}{2Q}} .
\end{equation}
(Proof in \cref{app:prob-bo}.)
\end{corollary}

In words, if a cumulative-regret bound holds, it translates into a guarantee on the best sequence the run emits (\cref{app:prob-bo}), and
subtracting the coverage and interpolation errors additionally requires the Lipschitz condition of \cref{prop:theory-prob-interp} to hold.

\begin{takeaway}{Takeaway}
Each query is an unbiased, conditionally sub-Gaussian observation of $F_\tau$ (\cref{prop:theory-prob-bo}), so standard bandit and Bayesian optimization guarantees apply when their remaining regularity assumptions hold, and those guarantees control the best sequence the run emits (\cref{cor:theory-prob-bo-best}).
\end{takeaway}

\section{Experiments}
\label{sec:experiments}
We evaluate BOReFT on two tasks, Semantle word search and molecular property optimization, and organize the experiments around the following questions. First, whether search over the learned code space finds high-scoring sequences unseen during representation training. Second, how that comparison holds when other methods share the same representation-training data. Third, what components of BOReFT's training contribute to its performance, and last, whether the coverage and interpolation errors are consistent with the theory.
We also provide additional analyses and details in \Cref{app:experiments}, including the effect on search quality of alternative search domain construction, decoding temperature, kernels, acquisition functions, and intervention rank.

\subsection{Tasks and evaluation protocol}
\label{sec:experiments-setup}

\paragraph{Semantle word search.}
In Semantle \citep{agarwal2025bopro}, the generator proposes a single word and receives the cosine similarity of that word to a hidden target under a frozen embedding model, Qwen3-Embedding-0.6B \citep{zhang2025qwen3embedding}, as the black-box function $f$.
BOReFT uses Llama-3.2-1B-Instruct \citep{meta2024llama32} both as the generator and as the source of $\phi$, sets rank $r=64$, and uses for representation training $3072$ words ($\sim$2 tokens/word) paired with synthetically-generated definitions from Claude Opus 5 \citep{anthropic2026claudeopus}.
We score five hidden targets each from the train and held-out sets, each with three different warmstart sets of ten words each, and a budget of $500$ evaluations. 
The complete setup and protocol are in \cref{app:experiments-models,app:experiments-protocols}.

\paragraph{Molecular property optimization.}
The second task asks for a new molecule as a SMILES string \citep{weininger1988smiles} with high predicted activity on a protein target, a standard goal in drug discovery.
We use DRD2, GSK3$\beta$, and JNK3 predictors \citep{huang2021tdc,gao2022sample} as black-box objective functions, where each score is a class-1 probability in $[0,1]$, and an invalid SMILES string receives $0$.
BOReFT uses MiST \citep{bran2025mist}, a chemistry-pretrained Qwen2.5-3B model, as the generator and Qwen3-Embedding-0.6B as $\phi$, sets rank $r=64$, and trains \underline{once} on $1024$ description--molecule pairs from ChEBI-20 \citep{edwards2021text2mol} from molecules within the 90th percentile of activity for any target (i.e., DRD2 $\le 0.029$, GSK3$\beta$ $\le 0.10$, JNK3 $\le 0.04$; $\sim$31 tokens/molecule).
The budget is $500$ evaluations, with five warmstart sets of ten molecules.

\paragraph{Baselines and SFT evaluation.}
We compare with LLM-search baselines OPRO \citep{yang2024opro}, BOPRO \citep{agarwal2025bopro}, MiGrATe \citep{phan2025migrate}, AutoDiscovery \citep{agarwal2025autodiscovery}, SDPO-TTT \citep{hubotter2026reinforcement}, and random repeated sampling. These cover prevalent search strategies, including in-context learning, Bayesian optimization, tree-search, and test-time training.
We also include as a reference discrete BO \citep{rankovic2025gollum} over frozen embeddings, which can only propose sequences from the training set.
Lastly, to strictly compare the quality of search, we also evaluate the baselines after supervised fine-tuning (SFT) on the same train set that BOReFT uses, which puts the methods on a shared data prior.
We do this by training a LoRA adapter \citep{hu2022lora} for each base model for ten epochs.
The search mechanism for each baseline is in \cref{app:experiments-baselines}.

\subsection{Results and analyses}
\label{sec:experiments-main}

\paragraph{Semantle.}
\Cref{fig:search-baselines} follows the best similarity found at each verification step on the held-out words, and \cref{tab:search-baselines} gives the final exact-match (EM) count, mean of the best-found similarity, and novelty rate (w.r.t. training) across 5 seeds.
BOReFT recovers $8/15$ held-out targets, at mean similarity $0.892$, with novelty rate $0.691$, while the next-best EM count is only $3/15$ before SFT alignment and $2/15$ after it.
Further, discrete BO finds 0 held-out words by construction, but it finds every training word in the fewest number of evaluations (\cref{tab:search-train}).
Fine-tuning the baselines on the training words raises the held-out EM and similarity scores of AutoDiscovery and random sampling, while still trailing BOReFT, but also lowers scores on both metrics for BOPRO and OPRO. We discuss reasons for baseline performance further in \cref{app:experiments-baseline-discussion}.

\begin{figure}[t]
\centering
\includegraphics[width=0.9\linewidth]{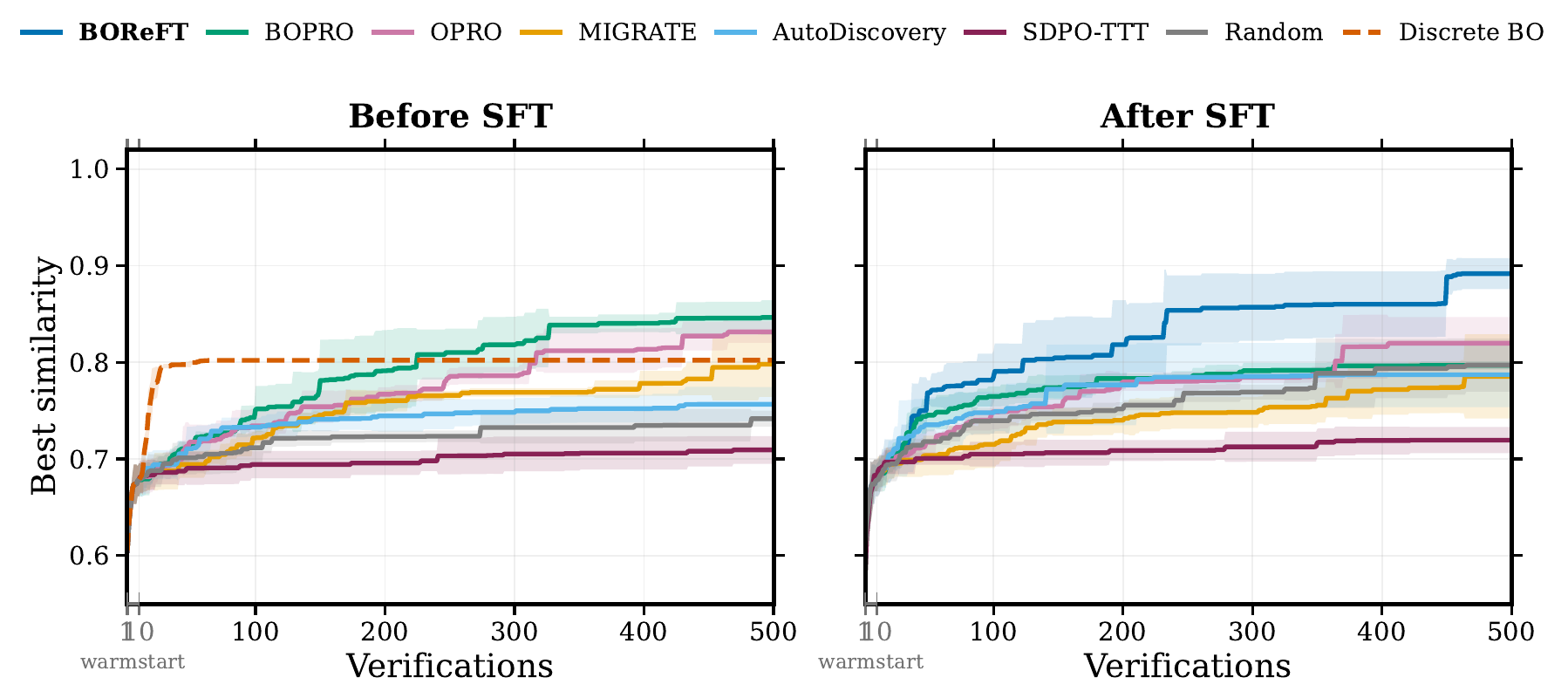}
\vspace{-0.5em}
\caption{\textbf{Best-so-far performance on held-out Semantle search.}
Best semantic similarity found at each verification step, averaged over 5 targets and 3 warmstart runs, with a band of ±1 SD.
We plot BOReFT only in the after-SFT block to group together methods that all use the training data.
}
\label{fig:search-baselines}
\end{figure}

\begin{table}[t]
\centering
\small
\setlength{\tabcolsep}{4pt}
\begin{tabular}{@{}lcccccc@{}}
\toprule
& \multicolumn{3}{c}{\textbf{Before SFT}} & \multicolumn{3}{c}{\textbf{After SFT}} \\
\cmidrule(lr){2-4} \cmidrule(lr){5-7}
\textbf{Method} & EM\,($\uparrow$) & Sim.\,($\uparrow$) & Novelty\,($\uparrow$) & EM\,($\uparrow$) & Sim.\,($\uparrow$) & Novelty\,($\uparrow$) \\
\midrule
\textit{Discrete BO} & \textit{0/15} & \textit{0.802} & \textit{0.000} & \textit{0/15} & \textit{0.802} & \textit{0.000} \\
\noalign{\vskip\aboverulesep}
\cdashline{1-7}
\noalign{\vskip\belowrulesep}
Random & $0/15$ & $0.710$ & $0.561$ & $1/15$ & $0.797$ & $0.119$ \\
SDPO-TTT & $0/15$ & $0.709$ & $0.679$ & $0/15$ & $0.719$ & $\underline{0.676}$ \\
AutoDiscovery & $0/15$ & $0.757$ & $\underline{0.736}$ & $1/15$ & $0.787$ & $0.154$ \\
MiGrATe & $1/15$ & $0.798$ & $0.630$ & $1/15$ & $0.785$ & $0.621$ \\
BOPRO & $\underline{3/15}$ & $\underline{0.846}$ & $\mathbf{0.798}$ & $0/15$ & $0.797$ & $0.203$ \\
OPRO & $\underline{3/15}$ & $0.831$ & $0.708$ & $\underline{2/15}$ & $\underline{0.820}$ & $0.212$ \\
\midrule
\textbf{BOReFT} & $\mathbf{8/15}$ & $\mathbf{0.892}$ & $0.691$ & $\mathbf{8/15}$ & $\mathbf{0.892}$ & $\mathbf{0.691}$ \\
\bottomrule
\end{tabular}
\caption{\textbf{Held-out Semantle search.}
Exact-match counts and mean best similarity over 15 runs (5 held-out targets and 3 warmstarts) with the rate of novelty w.r.t. the train set.
Bold and underline mark the best and second-best. BOReFT scores are repeated under before-SFT for clarity.}
\label{tab:search-baselines}
\end{table}

\paragraph{Molecular property optimization.}
\Cref{tab:molopt-search} reports for each property the mean and best activation probability found over 5 seeds for BOReFT and baselines, before and after SFT. Despite the percentile-90 training, BOReFT generalizes well beyond the scores it is trained on, and, notably, supports multiple task objectives simultaneously.
Before SFT, 
BOReFT leads JNK3 with mean $0.370$ and best score $0.940$, the highest by a margin $\ge 0.5$, and also finds the highest GSK3$\beta$ score of $0.670$. On DRD2, however, it trails OPRO, MiGrATe, and BOPRO.
After SFT, all baselines except MiGrATe show a drop in performance, promoting BOReFT to the top-2 on all properties.

\vspace{1em}
\begin{takeaway}{Takeaway}
On both Semantle and molecule property optimization, searching a learned manifold with Bayesian optimization is able to generalize beyond the task performance of its training data, and can outperform other LLM search strategies (ICL, TTT, tree-search). BOReFT's performance, however, remains coupled to the quality of training data (the representation gap of \cref{sec:theory-prob-setup}), which may prevent it from accessing high-scoring distributions, such as on DRD2, and motivates the need for continual learning to update the learned manifold as new observations are collected \citep{desanti2026actflow}, which we leave as future work.
\end{takeaway}

\begin{table}[t]
\centering
\small
\setlength{\tabcolsep}{3pt}
\begin{tabular}{@{}lcccccc@{}}
\toprule
& \multicolumn{3}{c}{\textbf{Before SFT}} & \multicolumn{3}{c}{\textbf{After SFT}} \\
\cmidrule(lr){2-4} \cmidrule(lr){5-7}
\textbf{Method} & DRD2\,($\uparrow$) & GSK3$\beta$\,($\uparrow$) & JNK3\,($\uparrow$) & DRD2\,($\uparrow$) & GSK3$\beta$\,($\uparrow$) & JNK3\,($\uparrow$) \\
\midrule
Random & $0.346/0.717$ & $0.310/0.590$ & $0.106/0.110$ & $0.235/0.403$ & $0.220/0.260$ & $0.094/0.120$ \\
SDPO-TTT$^\dagger$ & --\,/\,$\underline{0.981}$ & --\,/\,$0.280$ & --\,/\,$0.180$ & --\,/\,$0.780$ & --\,/\,$0.340$ & --\,/\,$0.190$ \\
AutoDiscovery & $0.534/0.929$ & $0.316/0.380$ & $0.158/0.160$ & $0.264/\underline{0.801}$ & $0.226/0.290$ & $0.094/0.160$ \\
MiGrATe & $\underline{0.880}/\mathbf{1.000}$ & $\underline{0.452}/0.580$ & $0.216/0.260$ & $\mathbf{0.844}/\mathbf{0.958}$ & $\mathbf{0.526}/\mathbf{0.710}$ & $\underline{0.218}/0.260$ \\
BOPRO & $0.253/\mathbf{1.000}$ & $\mathbf{0.470}/\underline{0.630}$ & $0.284/\underline{0.420}$ & $0.086/0.152$ & $0.322/0.520$ & $0.100/0.130$ \\
OPRO & $\mathbf{0.965}/\mathbf{1.000}$ & $0.402/0.430$ & $\underline{0.356}/0.410$ & $0.145/0.425$ & $0.275/0.360$ & $0.180/\underline{0.320}$ \\
\midrule
\textbf{BOReFT} & $0.645/0.782$ & $0.446/\mathbf{0.670}$ & $\mathbf{0.370}/\mathbf{0.940}$ & $\underline{0.645}/0.782$ & $\underline{0.446}/\underline{0.670}$ & $\mathbf{0.370}/\mathbf{0.940}$ \\
\bottomrule
\end{tabular}
\caption{\textbf{Molecular property optimization.}
Mean and best (\texttt{mean/max}) activation probability over 5 seeds at a budget of $500$ evaluations. 
Bold and underline mark the best and second-best. 
BOReFT scores are repeated under before-SFT for clarity.
$^\dagger$\textit{SDPO-TTT did not finish all five seeds.}}
\label{tab:molopt-search}
\end{table}

\begin{figure}[t]
    \centering
    \includegraphics[width=0.8\linewidth]{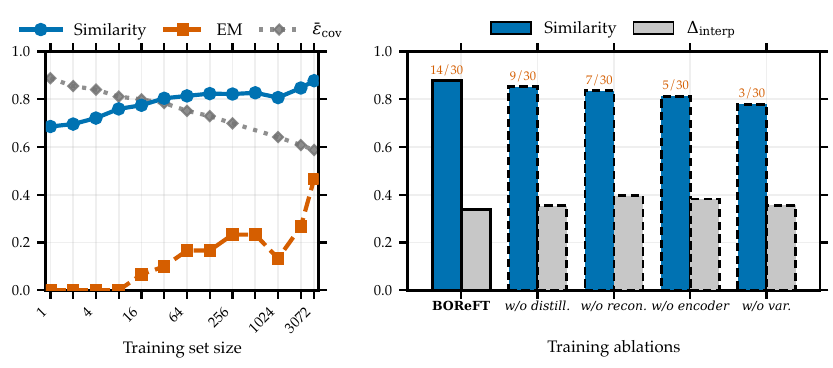}
    \vspace{-0.5em}
    \caption{\textbf{Training-set size and method ablations.}
    \textbf{Left:} Mean best similarity and EM on Semantle, along with median coverage error over 30 runs across train and test.
    \textbf{Right:} Mean best similarity and median interpolation distance $\Delta_\text{interp}$ on the same 30 runs, with EM counts in orange.}
    \label{fig:search-n}
\end{figure}

\vspace{2em}
\paragraph{What components of BOReFT help search?}
Full BOReFT finds $14/30$ hidden targets across train and test.
Removing self-distillation, reconstruction, or variational training lowers that count to $9/30$, $7/30$, and $3/30$, respectively, and removing the shared encoder lowers it to $5/30$ (\cref{fig:search-n}, right). This underscores the importance of each loss component, as well as the geometric bias from the encoder, in learning a suitable space for search
(see \cref{tab:search-breadth,fig:search-cosine}).

\paragraph{Do coverage and interpolation move with search?}
\Cref{prop:theory-prob-interp} bounds the representational gap by $L_f(\epscov + \epsint)$.
\Cref{fig:search-n}(left) therefore plots coverage error\footnote{We report an upper-bound on $\epscov$ (\cref{eq:theory-prob-epsint}) using the distance of a sequence to its nearest train embedding.} beside the search curves.
On $928$ held-out words, the median coverage error falls from $0.89$ at $N=1$ to $0.59$ at $N=3072$, while search improves correspondingly.
The interpolation error is the expected distance of one decoded sequence from the semantic interpolant $x_\alpha$.
\Cref{fig:search-n}(right) reports a related but smaller quantity $\Delta_\text{interp}$, the distance from the average decoded embedding to $x_\alpha$\footnote{By Jensen's inequality, it is at most $\epsint$ (\cref{app:prob-interp}).}.
We measure this halfway between pairs of training words, where the interpolant is farthest from either word. On 40 such pairs,
$\Delta_\text{interp}$ is $0.338$, while removing self-distillation, variational training, reconstruction, or the shared encoder all raise it to between $0.355$ and $0.396$ and reduce search quality (\cref{tab:search-breadth}).

\FloatBarrier
\section{Conclusion}
\label{sec:conclusion}
We introduce BOReFT, a generative search method for language models that first learns a low-dimensional continuous space of representation interventions, then uses Bayesian optimization to search that space a black-box objective. 
Our analysis characterizes when the space contains high-quality solutions and when adaptive search over it is well behaved, while experiments on Semantle and molecular property optimization show gains over baselines. 
Future work may consider how to expand the learned space with new observations.

\subsection*{AI use statement}
In this work, we used generative AI tools to generate the Semantle word definitions and the ChEBI-20 semantic categories, to assist with the theory proofs, and to help write and edit code.
All generated definitions, categories, proofs, and code were manually reviewed and edited before being incorporated.
We take responsibility for the final content of this work, including text, claims, and artifacts produced with the aid of generative AI.

\subsection*{Reproducibility statement}
Assumptions and complete proofs of the theoretical claims are given in \cref{sec:theory-prob} and \cref{app:proofs-prob}.
The method is specified in \cref{sec:method}.
Models, representation-training data, and evaluation protocols are described in \cref{app:experiments}.
Anonymized source code is included in the supplementary material.

\subsection*{Acknowledgements}
We thank Jeevana Kruthi Karnuthala and Mayank Gupta for their work during the early phase of this project, and Dhruvesh Patel and Kyle Richardson for providing feedback on early drafts.

\bibliography{refs}

\begin{thebibliography}{52}
\providecommand{\natexlab}[1]{#1}
\providecommand{\url}[1]{\texttt{#1}}
\expandafter\ifx\csname urlstyle\endcsname\relax
  \providecommand{\doi}[1]{doi: #1}\else
  \providecommand{\doi}{doi: \begingroup \urlstyle{rm}\Url}\fi

\bibitem[Agarwal et~al.(2025{\natexlab{a}})Agarwal, Arivazhagan, Das, Swamy, Khosla, and Gangadharaiah]{agarwal2025bopro}
Dhruv Agarwal, Manoj~Ghuhan Arivazhagan, Rajarshi Das, Sandesh Swamy, Sopan Khosla, and Rashmi Gangadharaiah.
\newblock Searching for optimal solutions with llms via bayesian optimization.
\newblock In \emph{International Conference on Learning Representations}, 2025{\natexlab{a}}.

\bibitem[Agarwal et~al.(2025{\natexlab{b}})Agarwal, Majumder, Adamson, Chakravorty, Gavireddy, Parashar, Surana, Mishra, McCallum, Sabharwal, and Clark]{agarwal2025autodiscovery}
Dhruv Agarwal, Bodhisattwa~Prasad Majumder, Reece Adamson, Megha Chakravorty, Satvika~Reddy Gavireddy, Aditya Parashar, Harshit Surana, Bhavana~Dalvi Mishra, Andrew McCallum, Ashish Sabharwal, and Peter Clark.
\newblock Autodiscovery: Open-ended scientific discovery via bayesian surprise.
\newblock In \emph{Advances in Neural Information Processing Systems}, 2025{\natexlab{b}}.

\bibitem[Aky{\"u}rek et~al.(2025)Aky{\"u}rek, Damani, Zweiger, Qiu, Guo, Pari, Kim, and Andreas]{akyurek2025surprising}
Ekin Aky{\"u}rek, Mehul Damani, Adam Zweiger, Linlu Qiu, Han Guo, Jyothish Pari, Yoon Kim, and Jacob Andreas.
\newblock The surprising effectiveness of test-time training for few-shot learning.
\newblock In \emph{Proceedings of the 42nd International Conference on Machine Learning}, volume 267 of \emph{Proceedings of Machine Learning Research}, pp.\  942--963. PMLR, 2025.

\bibitem[Ament et~al.(2023)Ament, Daulton, Eriksson, Balandat, and Bakshy]{ament2023unexpected}
Sebastian Ament, Samuel Daulton, David Eriksson, Maximilian Balandat, and Eytan Bakshy.
\newblock Unexpected improvements to expected improvement for {B}ayesian optimization.
\newblock In \emph{Advances in Neural Information Processing Systems (NeurIPS)}, 2023.

\bibitem[{Anthropic}(2026)]{anthropic2026claudeopus}
{Anthropic}.
\newblock System card: {Claude Opus} 5.
\newblock \url{https://www.anthropic.com/claude-opus-5-system-card}, 2026.
\newblock July 24.

\bibitem[Bran et~al.(2025)Bran, Xie, Pranesh, Meng, Nguyen, Goumaz, Segura, Xu, Zhou, Zhang, et~al.]{bran2025mist}
Andres~M Bran, Tong Xie, Shai Pranesh, Jeffrey Meng, Xuan~Vu Nguyen, Jeremy Goumaz, David~Ming Segura, Ruizhi Xu, Dongzhan Zhou, Wenjie Zhang, et~al.
\newblock Mist: Understanding the role of mid-stage scientific training in developing chemical reasoning models.
\newblock \emph{arXiv preprint arXiv:2512.21231}, 2025.

\bibitem[Brown et~al.(2024)Brown, Juravsky, Ehrlich, Clark, Le, R{\'e}, and Mirhoseini]{brown2024monkeys}
Bradley Brown, Jordan Juravsky, Ryan Ehrlich, Ronald Clark, Quoc~V. Le, Christopher R{\'e}, and Azalia Mirhoseini.
\newblock Large language monkeys: Scaling inference compute with repeated sampling.
\newblock \emph{arXiv preprint arXiv:2407.21787}, 2024.

\bibitem[Chen et~al.(2024)Chen, Chen, Goldstein, Huang, and Zhou]{chen2024instructzero}
Lichang Chen, Jiuhai Chen, Tom Goldstein, Heng Huang, and Tianyi Zhou.
\newblock {InstructZero}: Efficient instruction optimization for black-box large language models.
\newblock In \emph{International Conference on Machine Learning (ICML)}, 2024.

\bibitem[De~Santi et~al.(2026)De~Santi, Lee, Perez~Jensen, Protopapas, Tang, Liu, Chatterjee, Yue, and Krause]{desanti2026actflow}
Riccardo De~Santi, Bruce Lee, Cristian Perez~Jensen, Kimon Protopapas, Sophia Tang, Cheng-Hao Liu, Pranam Chatterjee, Yisong Yue, and Andreas Krause.
\newblock Active flow expansion for out-of-distribution discovery: from theory to molecules.
\newblock \emph{arXiv preprint arXiv:2606.08802}, 2026.

\bibitem[Deng et~al.(2025)Deng, Chang, and Chen]{deng2025distributionwise}
Chunyuan Deng, Ruidi Chang, and Hanjie Chen.
\newblock Learning distribution-wise control in representation space for language models.
\newblock In \emph{International Conference on Machine Learning (ICML)}, volume 267 of \emph{Proceedings of Machine Learning Research}, pp.\  13044--13068, 2025.

\bibitem[Edwards et~al.(2021)Edwards, Zhai, and Ji]{edwards2021text2mol}
Carl Edwards, ChengXiang Zhai, and Heng Ji.
\newblock {Text2Mol}: Cross-modal molecule retrieval with natural language queries.
\newblock In \emph{Proceedings of the 2021 Conference on Empirical Methods in Natural Language Processing}, pp.\  595--607. Association for Computational Linguistics, 2021.
\newblock \doi{10.18653/v1/2021.emnlp-main.47}.

\bibitem[Frazier(2018)]{frazier2018tutorial}
Peter~I. Frazier.
\newblock A tutorial on bayesian optimization.
\newblock \emph{arXiv preprint arXiv:1807.02811}, 2018.

\bibitem[Gao et~al.(2022)Gao, Fu, Sun, and Coley]{gao2022sample}
Wenhao Gao, Tianfan Fu, Jimeng Sun, and Connor~W. Coley.
\newblock Sample efficiency matters: A benchmark for practical molecular optimization.
\newblock In \emph{Advances in Neural Information Processing Systems (NeurIPS)}, 2022.

\bibitem[Geiger et~al.(2024)Geiger, Wu, Potts, Icard, and Goodman]{geiger2024das}
Atticus Geiger, Zhengxuan Wu, Christopher Potts, Thomas Icard, and Noah~D. Goodman.
\newblock Finding alignments between interpretable causal variables and distributed neural representations.
\newblock In \emph{Proceedings of the Third Conference on Causal Learning and Reasoning (CLeaR)}, 2024.

\bibitem[G{\'o}mez-Bombarelli et~al.(2018)G{\'o}mez-Bombarelli, Wei, Duvenaud, Hern{\'a}ndez-Lobato, S{\'a}nchez-Lengeling, Sheberla, Aguilera-Iparraguirre, Hirzel, Adams, and Aspuru-Guzik]{gomez2018automatic}
Rafael G{\'o}mez-Bombarelli, Jennifer~N. Wei, David Duvenaud, Jos{\'e}~Miguel Hern{\'a}ndez-Lobato, Benjam{\'i}n S{\'a}nchez-Lengeling, Dennis Sheberla, Jorge Aguilera-Iparraguirre, Timothy~D. Hirzel, Ryan~P. Adams, and Al{\'a}n Aspuru-Guzik.
\newblock Automatic chemical design using a data-driven continuous representation of molecules.
\newblock \emph{ACS Central Science}, 4\penalty0 (2):\penalty0 268--276, 2018.

\bibitem[Gurnee et~al.(2026)Gurnee, Sofroniew, Pearce, Piotrowski, Kauvar, Chen, Soligo, Bogdan, Ong, Wang, Thompson, Abrahams, Kantamneni, Ameisen, Batson, and Lindsey]{gurnee2026verbalizable}
Wes Gurnee, Nicholas Sofroniew, Adam Pearce, Mateusz Piotrowski, Isaac Kauvar, Runjin Chen, Anna Soligo, Paul Bogdan, Euan Ong, Rowan Wang, Ben Thompson, David Abrahams, Subhash Kantamneni, Emmanuel Ameisen, Joshua Batson, and Jack Lindsey.
\newblock Verbalizable representations form a global workspace in language models.
\newblock \emph{arXiv preprint arXiv:2607.15495}, 2026.

\bibitem[Hardt \& Sun(2024)Hardt and Sun]{hardt2024test}
Moritz Hardt and Yu~Sun.
\newblock Test-time training on nearest neighbors for large language models.
\newblock In \emph{The Twelfth International Conference on Learning Representations}, 2024.

\bibitem[Hinton et~al.(2015)Hinton, Vinyals, and Dean]{hinton2015distilling}
Geoffrey Hinton, Oriol Vinyals, and Jeff Dean.
\newblock Distilling the knowledge in a neural network.
\newblock \emph{arXiv preprint arXiv:1503.02531}, 2015.

\bibitem[Hu et~al.(2022)Hu, Shen, Wallis, Allen-Zhu, Li, Wang, Wang, and Chen]{hu2022lora}
Edward~J. Hu, Yelong Shen, Phillip Wallis, Zeyuan Allen-Zhu, Yuanzhi Li, Shean Wang, Lu~Wang, and Weizhu Chen.
\newblock {LoRA}: Low-rank adaptation of large language models.
\newblock In \emph{International Conference on Learning Representations (ICLR)}, 2022.

\bibitem[Huang et~al.(2021)Huang, Fu, Gao, Zhao, Roohani, Leskovec, Coley, Xiao, Sun, and Zitnik]{huang2021tdc}
Kexin Huang, Tianfan Fu, Wenhao Gao, Yue Zhao, Yusuf Roohani, Jure Leskovec, Connor~W. Coley, Cao Xiao, Jimeng Sun, and Marinka Zitnik.
\newblock Therapeutics data commons: Machine learning datasets and tasks for drug discovery and development.
\newblock In \emph{Proceedings of the Neural Information Processing Systems Track on Datasets and Benchmarks}, 2021.

\bibitem[Huben et~al.(2024)Huben, Cunningham, Smith, Ewart, and Sharkey]{cunningham2024sparse}
Robert Huben, Hoagy Cunningham, Logan~Riggs Smith, Aidan Ewart, and Lee Sharkey.
\newblock Sparse autoencoders find highly interpretable features in language models.
\newblock In \emph{International Conference on Learning Representations (ICLR)}, 2024.

\bibitem[H{\"u}botter et~al.(2026)H{\"u}botter, L{\"u}beck, Behric, Baumann, Bagatella, Marta, Hakimi, Shenfeld, Buening, Guestrin, and Krause]{hubotter2026reinforcement}
Jonas H{\"u}botter, Frederike L{\"u}beck, Lejs~Deen Behric, Anton Baumann, Marco Bagatella, Daniel Marta, Ido Hakimi, Idan Shenfeld, Thomas~Kleine Buening, Carlos Guestrin, and Andreas Krause.
\newblock Reinforcement learning via self-distillation.
\newblock In \emph{Forty-third International Conference on Machine Learning}, 2026.
\newblock URL \url{https://openreview.net/forum?id=QkfkxyRizZ}.

\bibitem[Kingma \& Welling(2014)Kingma and Welling]{kingma2014vae}
Diederik~P. Kingma and Max Welling.
\newblock Auto-encoding variational bayes.
\newblock In \emph{International Conference on Learning Representations (ICLR)}, 2014.

\bibitem[Kristiadi et~al.(2024)Kristiadi, Strieth-Kalthoff, Skreta, Poupart, Aspuru-Guzik, and Pleiss]{kristiadi2024sober}
Agustinus Kristiadi, Felix Strieth-Kalthoff, Marta Skreta, Pascal Poupart, Al{\'a}n Aspuru-Guzik, and Geoff Pleiss.
\newblock A sober look at {LLMs} for material discovery: Are they actually good for {B}ayesian optimization over molecules?
\newblock In \emph{International Conference on Machine Learning (ICML)}, 2024.

\bibitem[Locatello et~al.(2019)Locatello, Bauer, Lucic, R{\"a}tsch, Gelly, Sch{\"o}lkopf, and Bachem]{locatello2019challenging}
Francesco Locatello, Stefan Bauer, Mario Lucic, Gunnar R{\"a}tsch, Sylvain Gelly, Bernhard Sch{\"o}lkopf, and Olivier Bachem.
\newblock Challenging common assumptions in the unsupervised learning of disentangled representations.
\newblock In \emph{International Conference on Machine Learning (ICML)}, 2019.

\bibitem[Maus et~al.(2022)Maus, Jones, Moore, Kusner, Bradshaw, and Gardner]{maus2022local}
Natalie Maus, Haydn~T. Jones, Juston~S. Moore, Matt~J. Kusner, John Bradshaw, and Jacob~R. Gardner.
\newblock Local latent space {B}ayesian optimization over structured inputs.
\newblock In \emph{Advances in Neural Information Processing Systems (NeurIPS)}, 2022.

\bibitem[{Meta AI}(2024)]{meta2024llama32}
{Meta AI}.
\newblock Llama 3.2 model card.
\newblock Hugging Face model card, 2024.
\newblock URL \url{https://huggingface.co/meta-llama/Llama-3.2-1B-Instruct}.

\bibitem[Park et~al.(2024)Park, Choe, and Veitch]{park2024linear}
Kiho Park, Yo~Joong Choe, and Victor Veitch.
\newblock The linear representation hypothesis and the geometry of large language models.
\newblock In \emph{Proceedings of the 41st International Conference on Machine Learning}, volume 235 of \emph{Proceedings of Machine Learning Research}, pp.\  39643--39666. PMLR, 2024.

\bibitem[Phan et~al.(2025)Phan, Agarwal, Srinivas, Samulowitz, Kapanipathi, and McCallum]{phan2025migrate}
Peter Phan, Dhruv Agarwal, Kavitha Srinivas, Horst Samulowitz, Pavan Kapanipathi, and Andrew McCallum.
\newblock Migrate: Mixed-policy grpo for adaptation at test-time.
\newblock \emph{arXiv preprint arXiv:2508.08641}, 2025.

\bibitem[Polyanskiy \& Wu(2025)Polyanskiy and Wu]{polyanskiy2024information}
Yury Polyanskiy and Yihong Wu.
\newblock \emph{Information Theory: From Coding to Learning}.
\newblock Cambridge University Press, 2025.

\bibitem[Rankovi{\'c} \& Schwaller(2025)Rankovi{\'c} and Schwaller]{rankovic2025gollum}
Bojana Rankovi{\'c} and Philippe Schwaller.
\newblock {GOLLuM}: Gaussian process optimized {LLMs} -- reframing {LLM} finetuning through {B}ayesian optimization.
\newblock In \emph{ICLR Workshop on World Models}, 2025.

\bibitem[Rankovi{\'c} \& Schwaller(2026)Rankovi{\'c} and Schwaller]{rankovic2026ggollum}
Bojana Rankovi{\'c} and Philippe Schwaller.
\newblock Large language models as generative {B}ayesian policies.
\newblock In \emph{ICML Workshop on AI for Science}, 2026.
\newblock OpenReview: OSlXZGcur9.

\bibitem[Rasmussen(2003)]{rasmussen2003gaussian}
Carl~Edward Rasmussen.
\newblock Gaussian processes in machine learning.
\newblock In \emph{Summer school on machine learning}, pp.\  63--71. Springer, 2003.

\bibitem[Rimsky et~al.(2024)Rimsky, Gabrieli, Schulz, Tong, Hubinger, and Turner]{rimsky2024steering}
Nina Rimsky, Nick Gabrieli, Julian Schulz, Meg Tong, Evan Hubinger, and Alexander Turner.
\newblock Steering llama 2 via contrastive activation addition.
\newblock In \emph{Proceedings of the 62nd Annual Meeting of the Association for Computational Linguistics (Volume 1: Long Papers)}, pp.\  15504--15522. Association for Computational Linguistics, 2024.
\newblock \doi{10.18653/v1/2024.acl-long.828}.

\bibitem[Shahriari et~al.(2016)Shahriari, Swersky, Wang, Adams, and de~Freitas]{shahriari2016taking}
Bobak Shahriari, Kevin Swersky, Ziyu Wang, Ryan~P. Adams, and Nando de~Freitas.
\newblock Taking the human out of the loop: A review of bayesian optimization.
\newblock \emph{Proceedings of the IEEE}, 104\penalty0 (1):\penalty0 148--175, 2016.
\newblock \doi{10.1109/JPROC.2015.2494218}.

\bibitem[Siivola et~al.(2021)Siivola, Paleyes, Gonz{\'a}lez, and Vehtari]{siivola2021good}
Eero Siivola, Andrei Paleyes, Javier Gonz{\'a}lez, and Aki Vehtari.
\newblock Good practices for bayesian optimization of high dimensional structured spaces.
\newblock \emph{Applied AI Letters}, 2\penalty0 (2):\penalty0 e24, 2021.

\bibitem[Subramani et~al.(2022)Subramani, Suresh, and Peters]{subramani2022steering}
Nishant Subramani, Nivedita Suresh, and Matthew~E. Peters.
\newblock Extracting latent steering vectors from pretrained language models.
\newblock In \emph{Findings of the Association for Computational Linguistics: ACL 2022}, pp.\  566--581. Association for Computational Linguistics, 2022.
\newblock \doi{10.18653/v1/2022.findings-acl.48}.

\bibitem[Sun et~al.(2020)Sun, Wang, Liu, Miller, Efros, and Hardt]{sun2020testtime}
Yu~Sun, Xiaolong Wang, Zhuang Liu, John Miller, Alexei~A. Efros, and Moritz Hardt.
\newblock Test-time training with self-supervision for generalization under distribution shifts.
\newblock In \emph{Proceedings of the 37th International Conference on Machine Learning}, volume 119 of \emph{Proceedings of Machine Learning Research}, pp.\  9229--9248. PMLR, 2020.

\bibitem[Templeton et~al.(2024)Templeton, Conerly, Marcus, Lindsey, Bricken, Chen, Pearce, Citro, Ameisen, Jones, Cunningham, Turner, McDougall, MacDiarmid, Tamkin, Durmus, Hume, Mosconi, Freeman, Sumers, Rees, Batson, Jermyn, Carter, Olah, and Henighan]{templeton2024scaling}
Adly Templeton, Tom Conerly, Jonathan Marcus, Jack Lindsey, Trenton Bricken, Brian Chen, Adam Pearce, Craig Citro, Emmanuel Ameisen, Andy Jones, Hoagy Cunningham, Nicholas~L. Turner, Callum McDougall, Monte MacDiarmid, Alex Tamkin, Esin Durmus, Tristan Hume, Francesco Mosconi, C.~Daniel Freeman, Theodore~R. Sumers, Edward Rees, Joshua Batson, Adam Jermyn, Shan Carter, Chris Olah, and Tom Henighan.
\newblock Scaling monosemanticity: Extracting interpretable features from {Claude} 3 {Sonnet}.
\newblock \emph{Transformer Circuits Thread}, 2024.

\bibitem[Wang et~al.(2023)Wang, Wei, Schuurmans, Le, Chi, Narang, Chowdhery, and Zhou]{wang2023selfconsistency}
Xuezhi Wang, Jason Wei, Dale Schuurmans, Quoc Le, Ed~H. Chi, Sharan Narang, Aakanksha Chowdhery, and Denny Zhou.
\newblock Self-consistency improves chain of thought reasoning in language models.
\newblock In \emph{International Conference on Learning Representations (ICLR)}, 2023.

\bibitem[Weininger(1988)]{weininger1988smiles}
David Weininger.
\newblock Smiles, a chemical language and information system. 1. introduction to methodology and encoding rules.
\newblock \emph{Journal of Chemical Information and Computer Sciences}, 28\penalty0 (1):\penalty0 31--36, 1988.

\bibitem[Wilson et~al.(2016)Wilson, Hu, Salakhutdinov, and Xing]{wilson2016deep}
Andrew~Gordon Wilson, Zhiting Hu, Ruslan Salakhutdinov, and Eric~P. Xing.
\newblock Deep kernel learning.
\newblock In \emph{Proceedings of the 19th International Conference on Artificial Intelligence and Statistics}, volume~51 of \emph{Proceedings of Machine Learning Research}, pp.\  370--378. PMLR, 2016.

\bibitem[Wolpert(1996)]{wolpert1996lack}
David~H. Wolpert.
\newblock The lack of a priori distinctions between learning algorithms.
\newblock \emph{Neural Computation}, 8\penalty0 (7):\penalty0 1341--1390, 1996.

\bibitem[Wolpert \& Macready(1997)Wolpert and Macready]{wolpert1997no}
David~H. Wolpert and William~G. Macready.
\newblock No free lunch theorems for optimization.
\newblock \emph{IEEE Transactions on Evolutionary Computation}, 1\penalty0 (1):\penalty0 67--82, 1997.

\bibitem[Wu et~al.(2024)Wu, Arora, Wang, Geiger, Jurafsky, Manning, and Potts]{wu2024reft}
Zhengxuan Wu, Aryaman Arora, Zheng Wang, Atticus Geiger, Dan Jurafsky, Christopher~D. Manning, and Christopher Potts.
\newblock {ReFT}: Representation finetuning for language models.
\newblock In \emph{Advances in Neural Information Processing Systems (NeurIPS)}, 2024.

\bibitem[Wurgaft et~al.(2026)Wurgaft, Rager, Kowal, Shyam, Feucht, Bhalla, Haklay, Bigelow, Sarfati, McGrath, Lewis, Merullo, Goodman, Fel, Geiger, and Lubana]{wurgaft2026manifold}
Daniel Wurgaft, Can Rager, Matthew Kowal, Vasudev Shyam, Sheridan Feucht, Usha Bhalla, Tal Haklay, Eric Bigelow, Raphael Sarfati, Thomas McGrath, Owen Lewis, Jack Merullo, Noah~D. Goodman, Thomas Fel, Atticus Geiger, and Ekdeep~Singh Lubana.
\newblock Manifold steering reveals the shared geometry of neural network representation and behavior.
\newblock \emph{arXiv preprint arXiv:2605.05115}, 2026.

\bibitem[Yan et~al.(2026)Yan, Li, Hu, Wang, Cui, Qu, Cheng, and Zhang]{yan2026learning}
Jianhao Yan, Yafu Li, Zican Hu, Zhi Wang, Ganqu Cui, Xiaoye Qu, Yu~Cheng, and Yue Zhang.
\newblock Learning to reason under off-policy guidance.
\newblock \emph{Advances in Neural Information Processing Systems}, 38:\penalty0 117157--117186, 2026.

\bibitem[Yang et~al.(2024)Yang, Wang, Lu, Liu, Le, Zhou, and Chen]{yang2024opro}
Chengrun Yang, Xuezhi Wang, Yifeng Lu, Hanxiao Liu, Quoc~V. Le, Denny Zhou, and Xinyun Chen.
\newblock Large language models as optimizers.
\newblock In \emph{International Conference on Learning Representations}, 2024.

\bibitem[Yao et~al.(2023)Yao, Yu, Zhao, Shafran, Griffiths, Cao, and Narasimhan]{yao2023tree}
Shunyu Yao, Dian Yu, Jeffrey Zhao, Izhak Shafran, Thomas~L. Griffiths, Yuan Cao, and Karthik Narasimhan.
\newblock Tree of thoughts: Deliberate problem solving with large language models.
\newblock In \emph{Advances in Neural Information Processing Systems}, volume~36, 2023.

\bibitem[Zhang et~al.(2025)Zhang, Li, Long, Zhang, Lin, Yang, Xie, Yang, Liu, Lin, Huang, and Zhou]{zhang2025qwen3embedding}
Yanzhao Zhang, Mingxin Li, Dingkun Long, Xin Zhang, Huan Lin, Baosong Yang, Pengjun Xie, An~Yang, Dayiheng Liu, Junyang Lin, Fei Huang, and Jingren Zhou.
\newblock Qwen3 embedding: Advancing text embedding and reranking through foundation models.
\newblock \emph{arXiv preprint arXiv:2506.05176}, 2025.

\bibitem[Zhao et~al.(2026)Zhao, Xie, Liu, Huang, Pang, Chen, and Grover]{zhao2026selfdistilled}
Siyan Zhao, Zhihui Xie, Mengchen Liu, Jing Huang, Guan Pang, Feiyu Chen, and Aditya Grover.
\newblock Self-distilled reasoner: On-policy self-distillation for large language models.
\newblock In \emph{Forty-third International Conference on Machine Learning}, 2026.
\newblock URL \url{https://openreview.net/forum?id=Jpxfof0EaS}.

\bibitem[Zou et~al.(2023)Zou, Phan, Chen, Campbell, Guo, Ren, Pan, Yin, Mazeika, Dombrowski, Goel, Li, Byun, Wang, Mallen, Basart, Koyejo, Song, Fredrikson, Kolter, and Hendrycks]{zou2023repe}
Andy Zou, Long Phan, Sarah Chen, James Campbell, Phillip Guo, Richard Ren, Alexander Pan, Xuwang Yin, Mantas Mazeika, Ann-Kathrin Dombrowski, Shashwat Goel, Nathaniel Li, Michael~J. Byun, Zifan Wang, Alex Mallen, Steven Basart, Sanmi Koyejo, Dawn Song, Matt Fredrikson, J.~Zico Kolter, and Dan Hendrycks.
\newblock Representation engineering: A top-down approach to {AI} transparency.
\newblock \emph{arXiv preprint arXiv:2310.01405}, 2023.

\end{thebibliography}
\bibliographystyle{iclr2027_conference}

\appendix
\section{Proofs and experimental definitions}
\label{app:proofs-prob}

This appendix proves the results stated in \cref{sec:theory-prob}, gives the formal versions of the statements that \cref{sec:theory-prob-training} summarizes in words, and defines the diagnostics reported in \cref{sec:experiments}. 

\paragraph{Conventions.}
The search domain $\Bdom \subset \R^r$ is compact with nonempty interior. Each logit map $b \mapsto L_{p,u}(b)$ is continuous in $b$, because every frozen operation after the intervention site is continuous. Sampling at temperature $\tau > 0$ draws $s_t \sim p^\tau_b(\cdot \mid s_{<t}) \propto e^{L_{s_{<t},\cdot}(b)/\tau}$ and stops at the first $\eos$ or at a length cap $T_{\max}$. At $\tau = 0$, sampling reduces to greedy decoding, so $P^0_b(s)$ equals one when greedy decoding at $b$ emits $s$ and zero otherwise. Every target sequence is $\eos$-terminated: $s = (s_1, \dots, s_T)$ with $s_T = \eos$, $s_t \neq \eos$ for $t < T$, and $T \le T_{\max}$. With this convention the event that a run emits exactly $s$ coincides with the trajectory event, so $P^\tau_b(s) = \prod_{t=1}^{T} p^\tau_b(s_t \mid s_{<t})$ with the terminal $\eos$ factor included, and sequences truncated at the length cap are covered by treating $t = T_{\max}$ as a terminal step. Two consequences used below are that no target is a prefix of another, and that distinct terminated sequences correspond to disjoint emission events. Wherever a greedy argmax appears we break exact ties by a fixed total order on $\Vocab$, so that the greedy decode map $\PhiMap_0$ is defined everywhere. We call a code $b$ \emph{nondegenerate for $s$} if no logit difference $L_{s_{<t},s_t}(b) - L_{s_{<t},v}(b)$ with $v \neq s_t$ vanishes along the trajectory of $s$ at $b$, and statements about greedy decoding are made at such codes.

\paragraph{Reliable reachability.}
Some statements below concern what the learned space can emit with confidence rather than on average, so we record that notion here. At $\tau > 0$ every sequence has positive probability at every code, so the support of $P^\tau_b$ does not distinguish what the space can express. We instead fix a confidence level $\theta \in (0,1)$ and define the \emph{reachable set}
\begin{equation}
\label{eq:app-prob-reach}
\Reach_\tau(\theta) \;:=\; \big\{ s : \sup_{b \in \Bdom} P^\tau_b(s) \ge \theta \big\},
\end{equation}
the sequences that some code in the domain emits with probability at least $\theta$ in a single sample. When that set is nonempty we write $f^\star_{\mathrm{reach}}(\tau,\theta) := \max\{ f(s) : s \in \Reach_\tau(\theta) \}$ for the best score among its members. The two views are related. A code that emits a sequence $s$ with probability at least $\theta$ has $F_\tau(b) \ge \theta f(s) + (1-\theta) f_{\min}$, so taking the best such sequence gives
\begin{equation}
\label{eq:app-prob-reach-f}
F^\star_\tau \;\ge\; \theta\, f^\star_{\mathrm{reach}}(\tau,\theta) \;+\; (1-\theta)\, f_{\min} .
\end{equation}
The main text works with $F_\tau$ throughout, because that is the quantity a query observes.

\subsection{Reconstruction alone gives no held-out coverage \texorpdfstring{(\cref{prop:theory-prob-nfl})}{(Proposition A.1)}}
\label{app:prob-nfl}

\Cref{sec:theory-prob-training} states that fidelity to the training targets carries no information about held-out sequences, and this subsection makes that statement precise. Coverage is measured through the reachable set of \cref{eq:app-prob-reach}, because at positive temperature every held-out sequence has positive probability at every code and coverage in terms of support alone would be trivial. For a finite set $U$ of held-out targets we define the \emph{oracle coverage} $\chi_{\mathrm{oracle}}(U; \tau, \theta) := |\Reach_\tau(\theta) \cap U|/|U|$, which upper-bounds the fraction of $U$ that any search over $\Bdom$ could reliably elicit at confidence $\theta$.

\begin{proposition}[No coverage guarantee from reconstruction alone]
\label{prop:theory-prob-nfl}
Fix rank $r$, a compact $\Bdom$ with nonempty interior, distinct codes $\mu_1, \dots, \mu_n \in \Bdom$, targets $s_1, \dots, s_n$, and a finite held-out set $U$ of sequences disjoint from the targets. For any $\tau \ge 0$ and $\theta \in (1/2, 1)$, there exists no $g(n, r) > 0$ that lower-bounds $\chi_{\mathrm{oracle}}(U;\tau,\theta)$ over the class of decoders with smooth logit maps that satisfy $P^\tau_{\mu_i}(s_i) \ge \theta$ for all $i$.
\end{proposition}

The number of training targets and the rank therefore give no guarantee of held-out coverage. The construction uses the prefix tree of the training targets, so it applies to sequences that share prefixes and does not assume that the first token determines the output.

\begin{proof}[Proof of \cref{prop:theory-prob-nfl}]
Let $E$ be the set of prefix--token edges along the training sequences, and write $Z$ for the first edge at which each held-out sequence leaves that tree. Such an edge exists: a path that stayed on the tree through the terminal token would equal a training target. At that node the tree has some other continuation.

Fix a scale $a > 0$, small distinct tie-breaking offsets $c_v \in [0, a/4]$, and smooth bumps $\varphi_i(b) := \exp(-\norm{b - \mu_i}_2^2/\sigma^2)$. Define
\begin{equation}
\label{eq:app-prob-nfl-logits}
L_{p,v}(b) \;:=\; c_v \;-\; a\,\mathbf{1}\{(p,v) \in Z\} \;+\; 2a \sum_{i \,:\, (p,v) \in E_i} \varphi_i(b),
\end{equation}
where $E_i$ is the set of edges along training target $s_i$. These maps are smooth in $b$.

At $\mu_i$, the correct token of $s_i$ receives a bump of $2a$ at every step. Once $\sigma$ is small enough that every other bump is at most $1/(8n)$ there, each competitor is at most $a/2$, so the correct token leads by at least $3a/2$. A softmax step with that lead has conditional probability at least $1 - (K-1)e^{-(3a/2)/\tau}$ when $\tau > 0$, and a product over at most $T_{\max}$ steps is at least $\theta$ once $a \ge \tau\log\frac{T_{\max}(K-1)}{1-\theta}$. At $\tau = 0$ the same lead makes greedy decoding emit $s_i$. Thus every training target is reconstructed at confidence $\theta$.

A held-out leaving token is penalized by $a$ at every code and receives no bump, while the tree's alternative continuation is not penalized. The leaving token therefore trails by at least $3a/4$ everywhere, so its conditional probability is below $1/2$. The probability of the whole held-out sequence is at most that factor, hence below $\theta$ at every code. At $\tau = 0$, greedy decoding takes the alternative continuation and leaves the held-out trajectory. No positive function of $n$ and $r$ can therefore lower-bound $\chi_{\mathrm{oracle}}$ over this class.
\end{proof}

The decoder in the proof is a smooth map from codes to logits, a class larger than a transformer with a rank-$r$ intervention. The proposition says that smoothness together with high-confidence reconstruction leaves held-out coverage undetermined. The penalty is applied at the first prefix where a held-out sequence leaves the training tree, so the same construction covers targets that share a long prefix with a training sequence. 

\subsection{The prior and distillation terms \texorpdfstring{(\cref{lem:theory-prob-train})}{(Lemma A.2)}}
\label{app:prob-training}

This subsection states and proves the two consequences summarized in \cref{sec:theory-prob-training}. Throughout, $q_i = \mathcal{N}(\mu_i, \Sigma_i)$ is the Gaussian variational posterior for target $s_i$, with diagonal covariance in our implementation, $p_0 = \mathcal{N}(0, \sigma_0^2 I_r)$ is the shared prior, and $\Bdom$ is the axis-aligned bounding box of the means $\{\mu_1, \dots, \mu_n\}$. As in \cref{sec:method}, $P_b$ denotes the decoder distribution $P^\tau_b$ at the training temperature $\tau = 1$: teacher-forced reconstruction evaluates exactly $-\log P_b(s_i)$ under the termination convention, and the distillation term scores the student at the same temperature (the implementation exposes a separate softmax temperature for distillation, which we take at its default value of one). We use two standard facts \citep{polyanskiy2024information}: Pinsker's inequality $\TV(P,Q) \le \sqrt{\KL(P\|Q)/2}$ in natural logarithms, and joint convexity of both $\KL$ and $\TV$ in their arguments.

The first consequence follows from the closed form of the Gaussian divergence, and the second from Pinsker's inequality together with convexity of relative entropy in its second argument.

\begin{lemma}[Consequences of the prior and distillation terms]
\label{lem:theory-prob-train}
Write $\bar P_i(\cdot) := \E_{b \sim q_i}[P_b(\cdot)]$ for the student sequence distribution of target $i$ averaged over its posterior. Then:
\begin{enumerate}[leftmargin=*, itemsep=0pt, topsep=2pt, label=(\roman*)]
    \item if $\KL(q_i \| p_0) \le \kappa_i$ for every target $i$, then each mean satisfies $\norm{\mu_i}_2 \le \sigma_0 \sqrt{2\kappa_i}$, and with $\kappa_{\max} := \max_i \kappa_i$ the search domain satisfies $\Bdom \subseteq [-R, R]^r$ for $R := \sigma_0\sqrt{2\kappa_{\max}}$ and $\diam(\Bdom) \le 2\sigma_0\sqrt{2 r \kappa_{\max}}$;
    \item if $\E_{b \sim q_i}[\KL(Q_i \| P_b)] \le \eta_i$ for a target $i$, then $\TV(Q_i, \bar P_i) \le \sqrt{\eta_i/2}$. Hence $\bar P_i(A) \ge Q_i(A) - \sqrt{\eta_i/2}$ for every set $A$ of terminated sequences, for instance the sequences that express a given meaning, and $\tfrac{1}{2}\sum_{k \le m} |Q_i(A_k) - \bar P_i(A_k)| \le \sqrt{\eta_i/2}$ for any disjoint sets $A_1, \dots, A_m$.
\end{enumerate}
\end{lemma}

In words, part (i) says that the prior term controls the geometric extent of the domain that is later searched, which is the statement summarized in \cref{sec:theory-prob-training}. Part (ii) says that the distillation term makes the student's output distribution, averaged over the posterior, track the teacher's on every set of sequences simultaneously. A teacher that spreads its mass over several disjoint sets of sequences therefore forces that average to assign mass to each of them. The quantity $\eta_i$ is an expected sequence-level forward KL. The implemented distillation loss averages per-token divergences on a mixture of teacher samples, student rollouts, and the target, so the logged loss is not $\eta_i$.

\begin{proof}[Proof of \cref{lem:theory-prob-train}]
\emph{(i).} For the Gaussians $q_i = \mathcal{N}(\mu_i, \Sigma_i)$ and $p_0 = \mathcal{N}(0, \sigma_0^2 I_r)$ the divergence has the closed form
\begin{equation}
\label{eq:app-prob-gausskl}
\KL(q_i \,\|\, p_0)
\;=\;
\tfrac{1}{2}\Big( \norm{\mu_i}_2^2/\sigma_0^2 \;+\; \tr(\Sigma_i)/\sigma_0^2 \;-\; r \;-\; \log\det(\Sigma_i/\sigma_0^2) \Big),
\end{equation}
whose covariance contribution equals $\tfrac{1}{2}\sum_{j \le r} (\lambda_j - \log \lambda_j - 1)$ in terms of the eigenvalues $\lambda_j$ of $\Sigma_i/\sigma_0^2$ and is therefore nonnegative, since $x - \log x - 1 \ge 0$ for $x > 0$. Hence $\norm{\mu_i}_2^2 \le 2\sigma_0^2 \kappa_i$ for every target. Every coordinate of every mean then satisfies $|\mu_{ij}| \le \norm{\mu_i}_2 \le \sigma_0\sqrt{2\kappa_{\max}} = R$, so the bounding box of the means lies in $[-R, R]^r$, whose Euclidean diameter is $2R\sqrt{r}$, and the box's diameter is at most that. The same covariance term is nonnegative and diverges as any eigenvalue of $\Sigma_i/\sigma_0^2$ tends to $0$ or to $\infty$, so a small divergence also keeps the posterior spread comparable to the prior's.

\emph{(ii).} The averaged distribution $\bar P_i$ is a mixture of the $P_b$ over the mixing measure $q_i$, and the teacher $Q_i$ does not depend on $b$, so convexity of relative entropy in its second argument gives
\begin{equation}
\KL(Q_i \,\|\, \bar P_i) \;\le\; \E_{b \sim q_i}\big[\KL(Q_i \,\|\, P_b)\big] \;\le\; \eta_i .
\end{equation}
Pinsker's inequality then gives $\TV(Q_i, \bar P_i) \le \sqrt{\eta_i/2}$. The bound $\bar P_i(A) \ge Q_i(A) - \sqrt{\eta_i/2}$ is the definition of total variation applied to the event $A$. For disjoint sets $A_1, \dots, A_m$, the map sending a sequence to the index of the set containing it (with one extra index for sequences in none of them) is a deterministic function of the sample, so the induced distributions on indices are pushforwards of $Q_i$ and $\bar P_i$. Total variation cannot increase under a pushforward, and for distributions on a finite set it equals half the $\ell_1$ distance, which gives $\tfrac{1}{2}\sum_{k \le m} |Q_i(A_k) - \bar P_i(A_k)| \le \TV(Q_i, \bar P_i) \le \sqrt{\eta_i/2}$.
\end{proof}

\subsection{Proof of \texorpdfstring{\cref{prop:theory-prob-interp}}{Proposition 4.1}}
\label{app:prob-interp}

\begin{proof}[Proof of \cref{prop:theory-prob-interp}]
Fix a sequence $s$ and a weight $\alpha \in \mathcal{A}$, and abbreviate $\E_S$ for the expectation over $S \sim P^\tau_{b_\alpha}$. Since $F_\tau(b_\alpha) = \E_S[f(S)]$, the Lipschitz hypothesis applied inside the expectation gives
\begin{equation}
f(s) - F_\tau(b_\alpha)
\;=\;
\E_S\big[ f(s) - f(S) \big]
\;\le\;
L_f\, \E_S \norm{\phi(s) - \phi(S)}_2 .
\end{equation}
For every sequence $S$ the triangle inequality in the embedding space gives $\norm{\phi(s) - \phi(S)}_2 \le \norm{\phi(s) - x_\alpha}_2 + \norm{\phi(S) - x_\alpha}_2$, and taking expectations gives
\begin{equation}
\label{eq:app-prob-interp-two}
f(s) - F_\tau(b_\alpha)
\;\le\;
L_f \Big( \norm{\phi(s) - x_\alpha}_2 + \E_S \norm{\phi(S) - x_\alpha}_2 \Big) .
\end{equation}
Each $b_\alpha$ is a convex combination of the means $\mu_i$, and the axis-aligned bounding box $\Bdom$ of those means is convex and contains them, so $b_\alpha \in \Bdom$ and $F^\star_\tau \ge F_\tau(b_\alpha)$. The second term of \cref{eq:app-prob-interp-two} is at most $L_f \epsint$, because $\epsint$ is a supremum over $\mathcal{A}$, so
\begin{equation}
F^\star_\tau \;\ge\; f(s) - L_f \norm{\phi(s) - x_\alpha}_2 - L_f \epsint
\end{equation}
for every $\alpha \in \mathcal{A}$. The left-hand side does not depend on $\alpha$, so taking the supremum of the right-hand side over $\mathcal{A}$ replaces $\norm{\phi(s) - x_\alpha}_2$ by its infimum, which is $\epscov(s)$.
\end{proof}

The bound is loose by construction, since it passes through a worst-case Lipschitz constant and discards the direction in which the decoded semantics deviate from the interpolant. It is a decomposition of the representational gap into two measurable parts, not a tight estimate.

\paragraph{Remark (scores that are linear in the embedding).}
When the score reads the embedding linearly, the interpolation term can be replaced by a smaller quantity. Semantle is of this form: its embeddings are normalized, the hidden target $s^\star$ has embedding $x^\star := \phi(s^\star)$, and $f(s) = \langle x^\star, \phi(s) \rangle$, so $L_f = 1$ and $f^\star = 1$ by the Cauchy--Schwarz inequality. Writing
\begin{equation}
\label{eq:app-prob-mtau}
m_\tau(b) \;:=\; \E_{S \sim P^\tau_b}\big[\phi(S)\big]
\end{equation}
for the mean semantics of what a code emits, we have $F_\tau(b) = \langle x^\star, m_\tau(b) \rangle$. Adding and subtracting $\langle x^\star, x_\alpha \rangle$, and applying the Cauchy--Schwarz inequality to the second term, gives
\begin{equation}
\label{eq:app-prob-interp-lin}
1 - F^\star_\tau
\;\le\;
\Big( 1 - \sup_{\alpha \in \mathcal{A}} \langle x^\star, x_\alpha \rangle \Big)
\;+\;
\sup_{\alpha \in \mathcal{A}} \norm{m_\tau(b_\alpha) - x_\alpha}_2 .
\end{equation}
The interpolation term of \cref{eq:app-prob-interp-lin} is the displacement of the mean decoded semantics rather than the expected spread around the interpolant, and by Jensen's inequality it is the smaller of the two. It is the quantity that the interpolation paths of \cref{sec:experiments} report. The coverage term of \cref{eq:app-prob-interp-lin} is linear in $\alpha$, so its optimum on the simplex is attained at a vertex and reduces to the nearest training target. The distance-based coverage error of \cref{eq:theory-prob-epsint} can credit a mixture. The experiments report the vertex form.

\paragraph{Remark (the shared encoder as an inductive bias).}
Nothing in \cref{eq:method-loss} makes $\epsint$ small. An isotropic prior in particular does not guarantee that individual code directions correspond to distinct semantic attributes of the decoded text, and alignment of that kind requires additional inductive biases on the model or the data \citep{locatello2019challenging}. The shared encoder of \cref{sec:method}, which produces each posterior from a semantic encoding of a privileged description of its target, is an inductive bias of this kind, because it ties the geometry of the codes to the geometry of those inputs.

\paragraph{Measuring the two errors.}
\Cref{prop:theory-prob-interp} holds for every nonempty set $\mathcal{A}$ of interpolation weights. A smaller set lowers $\epsint$ and raises $\epscov$, so any such choice keeps the bound valid, and the two errors move in opposite directions as weights are added. Restricting coverage to the vertices, the nearest training embedding, can only increase $\epscov$. The protocol used in the experiments is in \cref{app:experiments-interp}.

\subsection{Proofs of \texorpdfstring{\cref{prop:theory-prob-bo,cor:theory-prob-bo-best}}{Proposition 4.2 and Corollary 4.3}}
\label{app:prob-bo}

This subsection proves the two properties of the expected sampled score $F_\tau(b) := \sum_s P^\tau_b(s) f(s)$ used in \cref{sec:theory-prob-bo}. A query is an unbiased, conditionally sub-Gaussian observation of $F_\tau$, and cumulative regret controls the best sequence sampled during the run. The sum defining $F_\tau$ runs over the finitely many terminated sequences of length at most $T_{\max}$, and the task score is bounded, $f(s) \in [f_{\min}, f_{\max}]$ with range $\Delta_f := f_{\max} - f_{\min}$.

\paragraph{Observation model (\texorpdfstring{\cref{prop:theory-prob-bo}}{Proposition 4.2}).}
Conditionally on the history $\mathcal{H}_{t-1}$ and the chosen query $b_t$, the sequence $S_t$ is drawn from $P^\tau_{b_t}$, so $\E[Y_t \mid \mathcal{H}_{t-1}, b_t] = \sum_s P^\tau_{b_t}(s) f(s) = F_\tau(b_t)$ directly from the definition. For the noise $\xi_t := Y_t - F_\tau(b_t)$, the observation $Y_t$ takes values in an interval of width $\Delta_f$ and has conditional mean $F_\tau(b_t)$, so Hoeffding's lemma gives
\begin{equation}
\E\big[e^{\lambda \xi_t} \,\big|\, \mathcal{H}_{t-1}, b_t\big] \;\le\; e^{\lambda^2 \Delta_f^2 / 8}
\qquad\text{for every } \lambda \in \R,
\end{equation}
so $\xi_t$ is conditionally $(\Delta_f/2)$-sub-Gaussian. The bound holds conditionally on the query, so it is unaffected by the fact that $b_t$ is chosen adaptively from the history.

\paragraph{From cumulative regret to the best sampled sequence.}
\begin{proof}[Proof of \cref{cor:theory-prob-bo-best}]
We write $R_Q := \sum_{t \le Q} \big( F^\star_\tau - F_\tau(b_t) \big)$. Averaging the identity $Y_t = F_\tau(b_t) + \xi_t$ over $t \le Q$ gives
\begin{equation}
\frac{1}{Q} \sum_{t \le Q} Y_t \;=\; F^\star_\tau \;-\; \frac{R_Q}{Q} \;+\; \frac{1}{Q} \sum_{t \le Q} \xi_t .
\end{equation}
The noise variables form a martingale difference sequence with respect to the filtration generated by the histories, and each is conditionally $(\Delta_f/2)$-sub-Gaussian by the observation-model paragraph above, so the Azuma--Hoeffding inequality gives, with probability at least $1 - \delta$,
\begin{equation}
\frac{1}{Q} \sum_{t \le Q} \xi_t \;\ge\; -\frac{\Delta_f}{2}\sqrt{\frac{2\log(1/\delta)}{Q}} \;=\; -\Delta_f \sqrt{\frac{\log(1/\delta)}{2Q}} .
\end{equation}
Since $\max_{t \le Q} Y_t \ge \frac{1}{Q}\sum_{t \le Q} Y_t$ and $Y_t = f(S_t)$, combining the two displays gives the stated inequality.
\end{proof}

The corollary takes a cumulative-regret bound as an input.
Combining it with \cref{prop:theory-prob-interp} further requires that the score be Lipschitz in the embedding, which holds for Semantle and remains an assumption for the molecular predictors.

\newpage
\section{Experimental details}
\label{app:experiments}

This appendix records implementation and protocol details summarized in \cref{sec:experiments}.

\subsection{Models and representation training}
\label{app:experiments-models}

\paragraph{Semantle.}
BOReFT uses Llama-3.2-1B-Instruct \citep{meta2024llama32} as the frozen generative model and as the source of the semantic representation used to predict code posteriors. Each target--description pair is formatted as \texttt{The meaning of '\{word\}' is: \{definition\}}. To construct the posterior representation, we run the sequence through the frozen backbone up to the penultimate decoder layer, apply a copied final decoder block and the model's final RMSNorm, and read the representation at the last token of the definition span. A rank-$8$ LoRA on the copied final block is trained jointly with the posterior projection and is used only to construct code posteriors; it is not used by the generative model during search. The representation-training set contains $3072$ target words paired with dictionary-like descriptions.
For this run the three terms in \cref{eq:method-loss} are weighted by $1$, $\beta = 1$, and $\lambda = 1$.

\paragraph{Molecular property optimization.}
BOReFT uses the chemistry-pretrained Qwen2.5-3B checkpoint released with MiST \citep{bran2025mist} as the frozen generative model; in our experiments this checkpoint is identified as \texttt{qwen\_pretranined\_v6}. Representation training uses ChEBI-20 \citep{edwards2021text2mol}, which pairs canonical SMILES strings with natural-language descriptions. Our prepared corpus contains $8000$ pairs, shuffled once, and the primary experiments use a random $N=1024$ subsample drawn with seed $42$.
For the checkpoint reported in \cref{tab:molopt-search}, reconstruction, the prior, and self-distillation are weighted by $1$, $\beta = 10^{-3}$, and $\lambda = 1$.

Generation follows the MiST completion format in \cref{app:experiments-prompts}. The prompt ends in \texttt{[START\_SMILES]}, and the gold continuation ends in \texttt{[END\_SMILES]}. The semantic encoder $\phi$ is the frozen Qwen3-Embedding-0.6B model \citep{zhang2025qwen3embedding}. Each molecule--description pair is formatted as \texttt{The molecule '\{SMILES\}' is: \{description\}}, and we L2-normalize the resulting embeddings before applying the learned posterior projection $g_\psi$.

\paragraph{Intervention site.}
\label{app:experiments-site}
On both tasks the edit in \cref{eq:method-intervention} is applied once, at rank $r=64$.
The layer is the first transformer block, and the edited vector is that block's residual-stream output.
Later blocks decode from the edited activation.
The token is a dedicated marker written at the start of the task instruction.
On Semantle the marker is \texttt{<|reserved\_special\_token\_0|>}, and it is the first token of the user turn inside the chat template.
On molecular optimization the marker is \texttt{<|boreft\_0|>}.
It is prefixed to the completion string, the tokenizer then adds its beginning-of-sequence token, and \texttt{[START\_SMILES]} remains the last token of the prompt.

On both tasks, $g_\psi$ is a multilayer perceptron. A linear layer of width $256$, a ReLU, and a layer norm are followed by a linear layer of width $128$ and a ReLU. Separate linear heads then predict the posterior mean and log-variance in $\R^r$.

\subsection{Prompts}
\label{app:experiments-prompts}

Semantle uses Llama-3.2-1B-Instruct, an instruction-tuned model. Reconstruction and self-distillation are user turns of its chat template, and generation continues from the assistant header. Molecular property optimization uses the MiST Qwen2.5-3B checkpoint as a completion model. Those prompts are plain prefixes, and the model continues the same string. Each generative instruction below is prefixed with the marker token of \cref{app:experiments-site}. Embedding strings are plain text on both tasks.

\paragraph{Reconstruction.}
The student is trained to continue a task prompt with the target. The description is reserved for the teacher and the embedding string.

Semantle, the user turn:
\begin{quote}
\small\ttfamily\raggedright
Generate an English word (only the word, without any decoration or formatting).
\end{quote}

Molecular optimization, the completion prefix. The gold continuation is the SMILES string followed by \texttt{[END\_SMILES]}.
\begin{quote}
\small\ttfamily\raggedright
Here is a valid SMILES string for a molecule (only the SMILES string; no additional text): [START\_SMILES]
\end{quote}

\paragraph{Self-distillation.}
The teacher is the unintervened frozen model. Its prompt includes the description of the training target. The student is still conditioned on the reconstruction prompt.

Semantle, the user turn. \texttt{\{definition\}} is the dictionary text of the target word.
\begin{quote}
\small\ttfamily\raggedright
Here is the definition of an English word: \{definition\}\\
Generate a word that matches this definition (only the word, without any decoration or formatting).
\end{quote}

Molecular optimization. \texttt{\{definition\}} is the molecule description.
\begin{quote}
\small\ttfamily\raggedright
Here is the description of a molecule: \{definition\}\\
A valid SMILES string for such a molecule (only the SMILES string; no additional text): [START\_SMILES]
\end{quote}

\paragraph{Embedding.}
These strings are the inputs used to predict code posteriors.

Semantle. The string is read by the copied final block of Llama-3.2-1B-Instruct, at the last token of the definition span.
\begin{quote}
\small\ttfamily\raggedright
The meaning of '\{word\}' is: \{definition\}
\end{quote}

Molecular optimization. Qwen3-Embedding-0.6B encodes the string, and we L2-normalize the result before $g_\psi$.
\begin{quote}
\small\ttfamily\raggedright
The molecule '\{SMILES\}' is: \{description\}
\end{quote}

\paragraph{Definition generation.}
The Semantle definitions paired with the representation-training words were written by Claude Opus~5 \citep{anthropic2026claudeopus}.

\paragraph{Search.}
Semantle decodes from the reconstruction prompt above. The molecular runs in \cref{tab:molopt-search} replace it with an objective line followed by the same completion prefix, including \texttt{[START\_SMILES]}. For DRD2 the prompt is
\begin{quote}
\small\ttfamily\raggedright
The task is to optimize for DRD2 binding.\\
Here is a valid SMILES string for a molecule (only the SMILES string; no additional text): [START\_SMILES]
\end{quote}
GSK3$\beta$ and JNK3 keep that second line. Their first lines are ``The task is to optimize for GSK3$\beta$ (GSK3B) inhibition.'' and ``The task is to optimize for JNK3 inhibition.''

\subsection{Evaluation protocols}
\label{app:experiments-protocols}

\paragraph{Semantle.}
We evaluate ten hidden targets: five sampled from the $3072$ representation-training words and five from the held-out reconstruction-test pool, with three independent runs per target. Target selection uses seed $42$, and these words are distinct from the evaluation targets used in BOPRO and MiGrATe. Each run has a budget of $500$ objective evaluations. For each target and seed, all methods receive the same ten labeled warm-start word--score pairs, drawn from the training vocabulary and excluding the hidden target. BOReFT uses greedy decoding during Semantle search.
The surrogate is a Gaussian process with an ARD Mat\'ern-$2.5$ kernel and log expected improvement, unless varied in \cref{app:experiments-kernel}.
Before the kernel, each code is normalized to the search box and passed through one linear layer of width $64$ followed by an ELU.
That feature map is trained jointly with the Gaussian process, the deep-kernel construction of \citet{wilson2016deep} in the form used by GOLLuM \citep{rankovic2025gollum}.
Molecular search uses the same surrogate.
Unless varied in \cref{app:experiments-search-domain,app:experiments-temperature}, BOReFT searches the axis-aligned bounding box of the learned posterior means in \cref{eq:method-box}.

\paragraph{Molecular property optimization.}
We evaluate DRD2, GSK3$\beta$, and JNK3 using the molecular property scores described in \cref{sec:experiments-setup}. Invalid SMILES receive score $0$. All three oracles are class-1 probabilities in $[0, 1]$ (sklearn 1.4+ no longer normalizes TDC's stored leaf counts, so we restore that normalization at load time). Each method receives a budget of $500$ property evaluations and is evaluated over five independent seeds. We draw ten Sobol points in the learned BOReFT domain, decode them once at temperature $1$, and use the resulting molecules as the common warm-start proposals; their scores are computed separately under each objective. BOReFT samples at temperature $1$ during molecular search. \Cref{fig:molopt-search} shows the before- and after-SFT best-so-far curves for the runs in \cref{tab:molopt-search}.

\begin{figure}[t]
\centering
\includegraphics[width=\linewidth]{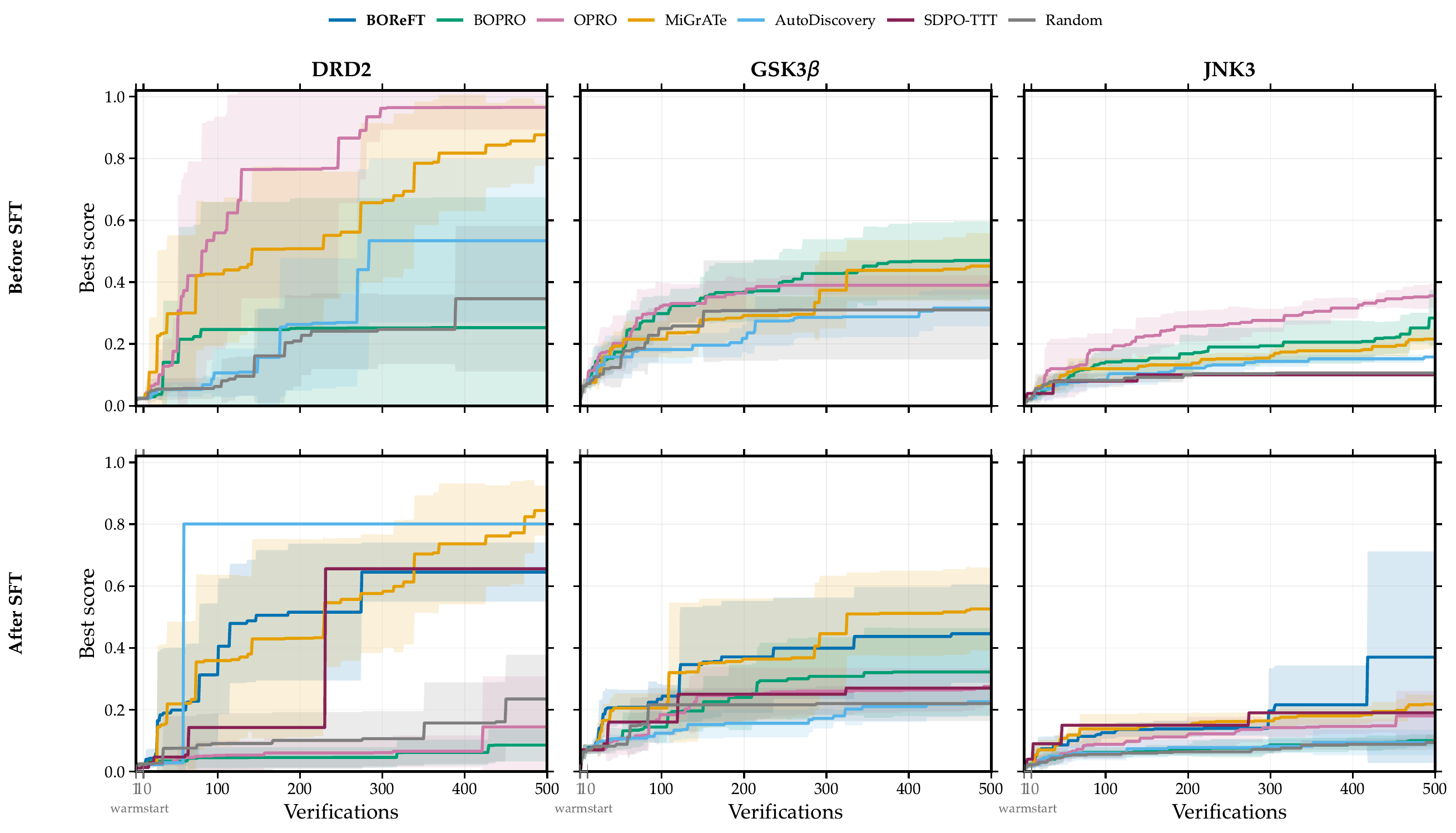}
\vspace{-0.5em}
\caption{\textbf{Search against baselines on molecular property optimization.}
Best-so-far property at each verification, mean $\pm$ 1 SD across seeds.
The top row is before SFT and the bottom row is after SFT.
The first ten evaluations are shared warm starts.
We plot BOReFT only after SFT; its scores are repeated in the before-SFT columns of \cref{tab:molopt-search}.
SDPO-TTT, AutoDiscovery, and MiGrATe are the no-instruction runs.
A curve with no band is a single finished seed.}
\label{fig:molopt-search}
\end{figure}

\subsection{Search baselines}
\label{app:experiments-baselines}

\Cref{tab:baseline-settings} lists the proposal rule and per-round query count used by each search method. Unless a row says otherwise, methods share the objective-evaluation budget and the same warm-start observations. On Semantle the base model is Llama-3.2-1B-Instruct; on molecular optimization it is the MiST chemistry-pretrained Qwen2.5-3B checkpoint of \cref{app:experiments-models}. Prompt-based methods sample at temperature $1$. Molecular completions use the MiST \texttt{[START\_SMILES]} / \texttt{[END\_SMILES]} format described above.

\begin{table}[t]
\centering
\small
\setlength{\tabcolsep}{4pt}
\begin{tabular}{llll}
\toprule
\textbf{Method} & \textbf{Update} & \textbf{Proposal rule} & \textbf{Scored / round} \\
\midrule
Random (base) & none & task prompt only & $1$ \\
Random (post-SFT) & none & same, after LoRA-SFT on the training set & $1$ \\
OPRO & none & scored-history in-context prompt & $1$ \\
BOPRO & GP over embeddings & OPRO prompt over $5$ nearest neighbors & $1$ \\
AutoDiscovery & UCB1 MCTS & OPRO prompt on the selected branch & $1$ of $8$ \\
MiGrATe & LoRA GRPO & mixed on-policy / greedy / neighborhood group & $4$ \\
SDPO-TTT & LoRA self-distillation & task prompt after a reverse-KL update & $1$ \\
Discrete BO & GP over the training set & acquisition over that finite pool & $1$ \\
BOReFT & GP over codes & decode from a queried code $b$ & $1$ \\
\bottomrule
\end{tabular}
\caption{\textbf{Search-method settings.}
All methods use the same base model within a task. AutoDiscovery expands eight completions per node and scores one; the rest remain as untried children. MiGrATe scores four new completions per round (two on-policy and two neighborhood); greedy group members are reused from history. BOPRO's GP is fit in the Qwen3 embedding space of observed solutions, not in the BOReFT code space.}
\label{tab:baseline-settings}
\end{table}

\subsection{Proposal behavior on Semantle}
\label{app:experiments-baseline-discussion}

\Cref{tab:search-baselines} shows that SDPO-TTT, AutoDiscovery, and MiGrATe recover at most one held-out target. This section describes the proposals behind those results. All three methods use the settings in \cref{tab:baseline-settings} and the same warm starts as the other Semantle runs. Averaged over the $30$ runs, the best warm-start similarity is $0.683$. Averages below are over those runs, and they exclude the warm-start proposals.

\paragraph{SDPO-TTT.}
Each proposal is sampled from the task description. That prompt does not list previously scored words. The reverse-KL update conditions the teacher on the most recent proposal and its score. Across these runs the most common proposal is the word ``cloud'', which is not a warm-start word. It accounts for $0.28$ of the first $30$ proposals and $0.32$ of the last $100$. In every run, at least $0.29$ of the proposals after the warm start are this word. The mean proposal similarity is $0.633$ over the first $50$ proposals and $0.630$ over the last $50$, below the warm-start best of $0.683$. Later proposals do not score higher than earlier ones. \Cref{tab:search-train} reports a best similarity of $0.763$ on training targets, and \cref{tab:search-baselines} reports $0.709$ on held-out targets, with no exact matches.

SDPO forms its teacher by conditioning the current model on feedback from the attempt, then distills that teacher into the student \citep{hubotter2026reinforcement}. The method is designed for rich environment feedback, such as a runtime error or a written judgment, which tells the teacher something the student did not see. When the environment returns only a scalar reward, the same work uses a successful attempt as feedback for the failures. Semantle returns only a scalar cosine similarity, and these runs never produce the hidden word for the teacher to condition on. The scalar attached to the latest word is the teacher's only additional context. The teacher distribution can therefore stay close to the unconditioned proposal distribution, and the distillation target supplies little additional supervision. This reading is consistent with proposal similarities that stay flat across the budget.

\paragraph{AutoDiscovery.}
Each proposal is drawn from a prompt built on the selected branch. The prompt lists scored words on that branch and asks for a word with a higher score. The mean proposal similarity is $0.624$ over the first $50$ proposals and $0.624$ over the last $50$, below the mean best warm-start similarity. A proposal exceeds the best score observed so far on $0.006$ of search steps. The repetition rate in \cref{tab:search-repetition} is $0.364$, so most proposals are new words. No run records an exact match. The final similarities are $0.772$ on training targets and $0.757$ on held-out targets.

Exploration is strong relative to exploitation. The search keeps introducing new words, and those words rarely improve on the best score already observed. That balance can suit the setting AutoDiscovery was introduced for, open-ended search over scientific hypotheses, where continued exploration is part of the goal \citep{agarwal2025autodiscovery}. Semantle asks for one hidden word, so search needs proposals that improve on the best score already observed. The runs here already move toward exploitation relative to the published defaults. The exploration constant is $C=0.5$ rather than $C=2.0$, and the prompt includes $20$ nodes from the selected branch rather than $3$. Even so, the mean proposal similarity does not rise. A different acquisition function, or a further change in these hyperparameters, might exploit high-scoring guesses more directly. We evaluate one fixed configuration and do not sweep those choices.

\paragraph{MiGrATe.}
MiGrATe updates a LoRA adapter after each group of proposals. The group mixes samples from the task prompt with variants of high-scoring words already observed \citep{phan2025migrate}. The mean proposal similarity rises from $0.637$ over the first $50$ proposals to $0.742$ over the last $50$. The fraction of proposals equal to ``cloud'' falls from $0.27$ in the first $30$ proposals to $0.03$ in the last $100$. The test-time GRPO update does change the proposal distribution on this task. Exact-match recovery remains low, at $0/15$ on training targets and $1/15$ on held-out targets, and the repetition rate is $0.678$.

The same update makes test-time search slower than BOReFT. Each round samples a group and takes a gradient step on the adapter. BOReFT fits a Gaussian process and decodes one code, with no gradient step. On these runs the mean elapsed time is $0.35$ hours for MiGrATe and $0.10$ hours for BOReFT, so MiGrATe takes about $3.5\times$ as long at the same budget of $500$ scores. The comparison is search time only. It does not include the cost of BOReFT's representation training.

\subsection{Search on training targets}
\label{app:experiments-sft-search}

\Cref{tab:search-train} reports exact match and mean best similarity when the hidden word is one of the five targets drawn from the representation-training set. The protocol matches \cref{tab:search-baselines}: three seeds, budget $500$, and the same ten warm starts. The after-training columns use one LoRA adapter, trained by supervised fine-tuning on the $3072$ representation-training words for 10 epochs, and shared by Random, SDPO-TTT, AutoDiscovery, MiGrATe, BOPRO, and OPRO. BOReFT does not use that adapter, so its entries repeat.

Held-out results for these runs are in \cref{tab:search-baselines}. None of the fine-tuned baselines matches BOReFT's $8/15$ held-out exact matches. On training targets, discrete BO reaches $15/15$ because those words are its candidate set. After fine-tuning, OPRO's training-target exact-match count falls from $7/15$ to $4/15$, while random sampling rises from $0/15$ to $3/15$.

Novelty on the pooled runs drops for AutoDiscovery, BOPRO, and OPRO, from $0.736$, $0.798$, and $0.708$ before fine-tuning to $0.154$, $0.203$, and $0.212$ after it (\cref{tab:search-baselines}). More of their proposals are words from the representation-training set. SDPO-TTT and MiGrATe stay near the base novelty rates.

\begin{table}[t]
\centering
\small
\setlength{\tabcolsep}{4pt}
\begin{tabular}{lcccc}
\toprule
& \multicolumn{2}{c}{\textbf{Before SFT}} & \multicolumn{2}{c}{\textbf{After SFT}} \\
\cmidrule(lr){2-3} \cmidrule(lr){4-5}
\textbf{Method} & Exact$\uparrow$ & Sim.$\uparrow$ & Exact$\uparrow$ & Sim.$\uparrow$ \\
\midrule
\textit{Discrete BO} & \textit{15/15} & \textit{1.000} & -- & -- \\
\noalign{\vskip\aboverulesep}
\cdashline{1-5}
\noalign{\vskip\belowrulesep}
Random & $0/15$ & $0.758$ & $3/15$ & $0.825$ \\
SDPO-TTT & $0/15$ & $0.763$ & $0/15$ & $0.766$ \\
AutoDiscovery & $0/15$ & $0.772$ & $2/15$ & $0.811$ \\
MiGrATe & $0/15$ & $0.788$ & $0/15$ & $0.782$ \\
BOPRO & $5/15$ & $\underline{0.878}$ & $\underline{5/15}$ & $0.857$ \\
OPRO & $\mathbf{7/15}$ & $\mathbf{0.886}$ & $4/15$ & $\underline{0.858}$ \\
\midrule
\textbf{BOReFT} & $\underline{6/15}$ & $0.863$ & $\mathbf{6/15}$ & $\mathbf{0.863}$ \\
\bottomrule
\end{tabular}
\caption{\textbf{Semantle search on training targets, before and after supervised fine-tuning.}
The table reports exact-match counts and mean best semantic similarity on five targets from the representation-training set, over fifteen runs.
The left block uses the frozen generator.
The right block uses the LoRA adapter of \cref{app:experiments-sft-search}.
BOReFT does not use that adapter.
Discrete BO searches the training vocabulary, so the right block has no entry for it.
In each column, bold marks the best value and underline marks the second best.
The italic row is a reference, and it is not included in that ranking.}
\label{tab:search-train}
\end{table}

\subsection{Proposal repetition}
\label{app:experiments-redundancy}

\Cref{tab:search-repetition} collects the repetition rate for the main comparisons on both tasks.
A proposal counts as a repeat when the same solution already occurred earlier in that run, including during the warm start.
Warm-start proposals are excluded from the denominator.
On Semantle we lower-case each word and collapse repeated whitespace, and the rate pools every post-warm-start proposal from the thirty runs on the training and held-out targets.
On the molecular task, valid SMILES strings are canonicalized before this comparison, so two strings for the same molecule are not counted as distinct. Invalidity is tracked separately from repetition.
Molecular rates pool the finished seeds on DRD2, GSK3$\beta$, and JNK3 from the runs in \cref{tab:molopt-search}.
Novelty, which asks whether a proposal falls outside the representation-training set, remains in \cref{tab:search-baselines} and is not repeated here.

\begin{table}[t]
\centering
\small
\setlength{\tabcolsep}{5pt}
\begin{tabular}{lcccc}
\toprule
& \multicolumn{2}{c}{\textbf{Semantle}} & \multicolumn{2}{c}{\textbf{Molecules}} \\
\cmidrule(lr){2-3} \cmidrule(lr){4-5}
\textbf{Method} & Before SFT & After SFT & Before SFT & After SFT \\
\midrule
\textit{Discrete BO} & \textit{0.000} & \textit{0.000} & -- & -- \\
\noalign{\vskip\aboverulesep}
\cdashline{1-5}
\noalign{\vskip\belowrulesep}
Random & $0.791$ & $\mathbf{0.201}$ & $0.212$ & $\underline{0.068}$ \\
SDPO-TTT & $0.750$ & $0.747$ & $0.176$ & $0.183$ \\
AutoDiscovery & $\mathbf{0.364}$ & $\underline{0.218}$ & $\underline{0.075}$ & $\mathbf{0.066}$ \\
MiGrATe & $0.678$ & $0.650$ & $0.112$ & $0.102$ \\
BOPRO & $0.527$ & $0.342$ & $\mathbf{0.039}$ & $0.094$ \\
OPRO & $0.540$ & $0.368$ & $0.223$ & $0.541$ \\
\midrule
\textbf{BOReFT} & $\underline{0.481}$ & $0.481$ & $0.133$ & $0.133$ \\
\bottomrule
\end{tabular}
\caption{\textbf{Proposal repetition before and after supervised fine-tuning.}
Each entry is the fraction of post-warm-start proposals that repeat an earlier solution in the same run.
Lower is better.
Semantle pools the training-target and held-out runs.
Molecular rates pool finished seeds on DRD2, GSK3$\beta$, and JNK3.
SDPO-TTT before SFT is the one finished JNK3 seed.
After SFT, SDPO-TTT is one finished seed on each property, and AutoDiscovery omits four unfinished DRD2 seeds.
Bold and underline mark the best and second-best rate in each column.
The italic row is a reference, and it is not ranked.
BOReFT does not use the fine-tuning adapter, so its entries repeat.}
\label{tab:search-repetition}
\end{table}

\subsection{Search-domain construction}
\label{app:experiments-search-domain}

The primary BOReFT runs search the axis-aligned bounding box (AABB) of the learned posterior means, \cref{eq:method-box}.
The variational posterior also produces a per-target standard deviation $\sigma_i$, so the same checkpoint admits other compact domains that still cover those means.
We evaluate two alternatives on Semantle, keeping the checkpoint, warm starts, budget, and acquisition unchanged.

The first generalizes the AABB to smoothly expand each coordinate proportional to its standard deviation,
\begin{equation}
\label{eq:app-box-std}
\prod_{j=1}^{r}
\left[
\min_i\bigl(\mu_{ij}-k\sigma_{ij}\bigr),\;
\max_i\bigl(\mu_{ij}+k\sigma_{ij}\bigr)
\right],
\end{equation}
where $k=0$ recovers the original bounds.
The second replaces the box with a covering ellipsoid \citep{siivola2021good} of the same means.
That ellipsoid is the Mahalanobis ball of their sample covariance, scaled so that every $\mu_i$ lies inside, and the acquisition is optimized over this convex set.

\Cref{tab:search-domain} reports the results together with each domain's Lebesgue volume relative to the mean box $V_0$ and the repetition rate.

Expanding the box by $k=0.1$ multiplies the volume by $4.71{\times}10^{5}$.
Exact-match recovery falls to $3/15$ on training targets and $4/15$ on held-out targets, and the repetition rate rises to $0.627$ and $0.577$.
At $k=0.5$ the volume is $1.08{\times}10^{21}$ times larger, no run recovers an exact match, and the repetition rate is $0.851$ on training targets and $0.844$ on held-out targets.
A full standard-deviation expansion ($k=1$) yields a volume $7.31{\times}10^{32}$ times that of the mean box and likewise recovers no exact matches.
Mean best similarity falls to $0.726$ on training targets and $0.707$ on held-out targets, which is comparable to random sampling from the frozen base model in \cref{tab:search-baselines}.
The repetition rate is then $0.943$ on training targets and $0.913$ on held-out targets.
On this checkpoint, the typical posterior standard deviation is on the order of the coordinate-wise span of the means, so even a modest $k$ in \cref{eq:app-box-std} enlarges the domain substantially.

The covering ellipsoid has smaller volume than the mean box ($4.08{\times}10^{-4}\,V_0$) and remains closer to it in search performance.
It finds $5/15$ training targets and $5/15$ held-out targets, against $6/15$ and $8/15$ for the mean box, with mean similarities $0.873$ and $0.852$ and repetition rates $0.594$ and $0.605$.
On this protocol, replacing the box with that covering ellipsoid does not improve exact-match recovery, while enlarging the box to include posterior spread reduces exact-match recovery and raises the repetition rate.

\begin{table}[t]
\centering
\small
\setlength{\tabcolsep}{3.5pt}
\begin{tabular}{lccccccc}
\toprule
& & \multicolumn{3}{c}{\textbf{Train}} & \multicolumn{3}{c}{\textbf{Held out}} \\
\cmidrule(lr){3-5} \cmidrule(lr){6-8}
\textbf{Search domain} & $V/V_0$ & Exact$\uparrow$ & Sim.$\uparrow$ & Rep.$\downarrow$ & Exact$\uparrow$ & Sim.$\uparrow$ & Rep.$\downarrow$ \\
\midrule
AABB & & & & & & & \\
\quad \textit{at $k=0$} & $1$ & $6/15$ & $0.863$ & $0.505$ & $8/15$ & $0.892$ & $0.452$ \\
\quad \textit{at $k=0.1$} & $4.71{\times}10^{5}$ & $3/15$ & $0.819$ & $0.627$ & $4/15$ & $0.836$ & $0.577$ \\
\quad \textit{at $k=0.5$} & $1.08{\times}10^{21}$ & $0/15$ & $0.758$ & $0.851$ & $0/15$ & $0.746$ & $0.844$ \\
\quad \textit{at $k=1$} & $7.31{\times}10^{32}$ & $0/15$ & $0.726$ & $0.943$ & $0/15$ & $0.707$ & $0.913$ \\
Covering ellipsoid & $4.08{\times}10^{-4}$ & $5/15$ & $0.873$ & $0.594$ & $5/15$ & $0.852$ & $0.605$ \\
\bottomrule
\end{tabular}
\caption{\textbf{Semantle search under alternative domains.}
Exact matches and mean best semantic similarity, using the protocol of \cref{tab:search-baselines} on the same checkpoint.
Repetition ($\downarrow$) uses the definition in \cref{tab:search-repetition}, pooled within each split: the fraction of post-warm-start proposals whose solution repeats an earlier proposal in the same run.
$V/V_0$ is the Lebesgue volume of the search domain relative to the mean box ($k=0$).
The axis-aligned bounding boxes follow \cref{eq:app-box-std}, and the covering ellipsoid is as described in \cref{app:experiments-search-domain} over the posterior means.}
\label{tab:search-domain}
\end{table}

\subsection{Geometric prior}
\label{app:experiments-encoder-pca}

The shared map $g_\psi$ predicts each code posterior from a frozen semantic representation of the representation-training input.
We call the geometry of these representations the geometric prior.
\Cref{fig:encoder-pca} shows the first two principal components of that prior.
The top row is the ChEBI-20 set of molecules with their descriptions, and the bottom row is the Semantle set of words with their definitions.
Within each row, points that share a semantic category use the same color.
The left column is Qwen3-Embedding-0.6B.
That encoder supplies $\phi$ on molecular optimization, and it is also the model used for the Semantle task score.
The right column is the task language model: the chemistry-pretrained MiST checkpoint on molecular optimization, and Llama-3.2-1B-Instruct on Semantle.
The Llama states are read from the last layer at the last instruction token, which is how $\phi$ is constructed on Semantle.

\begin{figure*}[t]
\centering
\includegraphics[width=\textwidth]{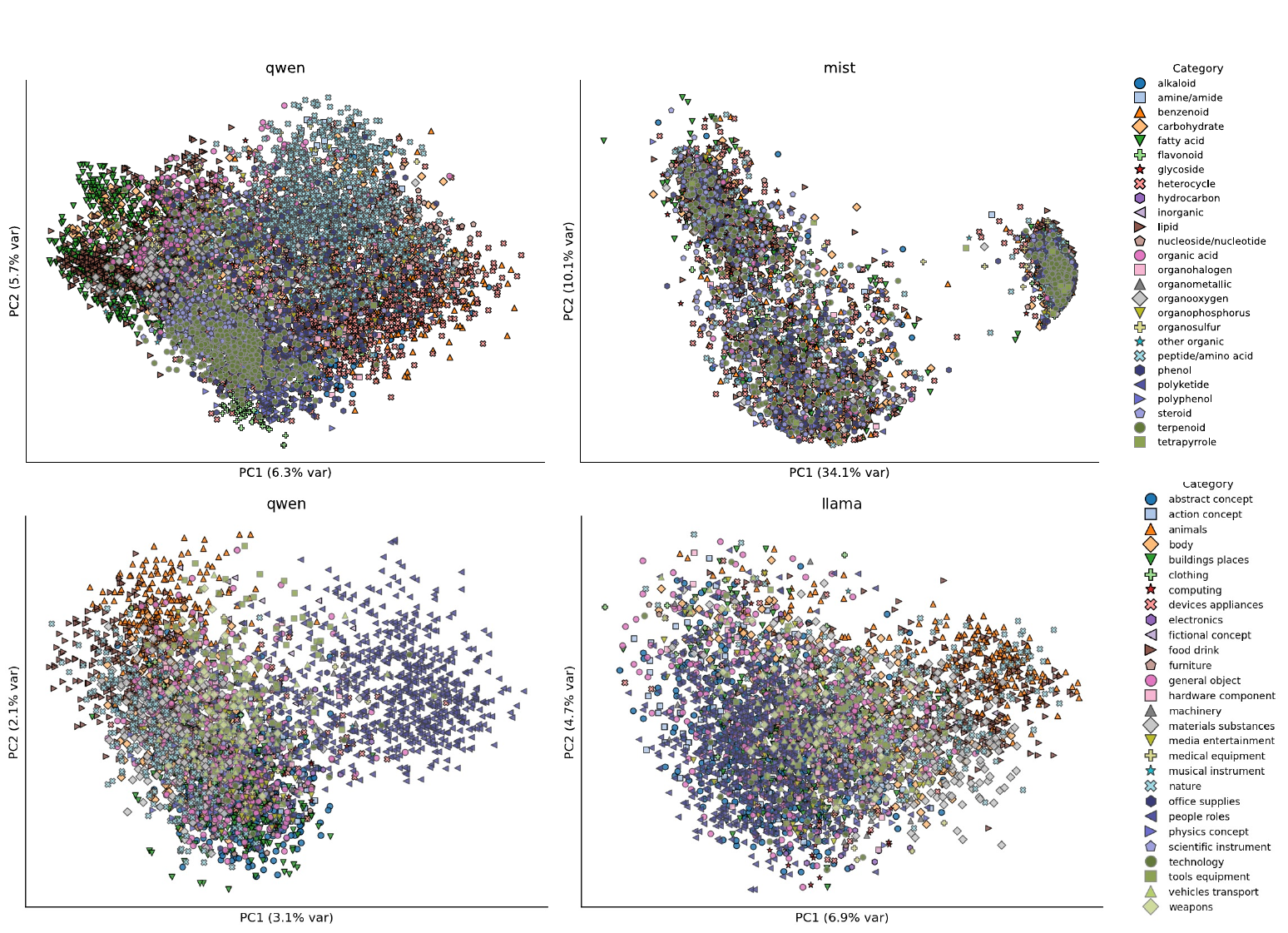}
\caption{\textbf{Geometric prior via semantic representations.}
First two principal components of the frozen semantic representations that supply the geometric prior.
The top row is molecular optimization and the bottom row is Semantle.
The left column is Qwen3-Embedding-0.6B.
The right column is the MiST chemistry-pretrained checkpoint on the top row and Llama-3.2-1B-Instruct on the bottom row.
Points are colored by semantic category, with one palette per row.
These are the representations passed to $g_\psi$.
The learned Semantle codes are shown separately in \cref{fig:interp-code-pca}.}
\label{fig:encoder-pca}
\end{figure*}

\subsection{Interpolation in the learned code space}
\label{app:experiments-interp}

\Cref{fig:interp-code-pca} shows qualitative interpolations in a PCA projection of the $3072$ Semantle training posterior means.
Paths run from ``research'' to four distant anchors, and temperature-sampled decodes along each path follow distinct semantic transitions.

To measure interpolation across the space, we decode temperature-$1$ samples along linear paths between $40$ training-target pairs at the mid-path weights $t\in\{0.4,0.5,0.6\}$.
At each weight we take the mean decoded embedding and its Euclidean distance to the semantic interpolant $x_\alpha$ of \cref{eq:theory-prob-interp}.
\Cref{tab:search-breadth} reports the median of these distances, written $\Delta_\text{interp}$.
An endpoint-only reference that always emits the nearer of the two training words, with no interpolation, has $\Delta_\text{interp}=0.365$.
BOReFT obtains $0.338$.
Removing self-distillation or variational training yields $0.355$ and $0.356$.
Removing reconstruction or the shared encoder yields $0.396$ and $0.382$, both larger than the endpoint-only reference.

\subsection{Effect of temperature on search}
\label{app:experiments-temperature}

The primary Semantle protocol decodes greedily ($T=0$).
We also search the same mean box at sampling temperatures $T\in\{0.25,0.50,0.75\}$, keeping the checkpoint, warm starts, budget, and acquisition unchanged.

\Cref{tab:search-temperature} reports the results.
Raising the temperature lowers the repetition rate, from $0.505$ and $0.452$ under greedy decoding to $0.234$ and $0.231$ at $T=0.75$.
Exact-match recovery is highest at $T=0$, with $6/15$ training targets and $8/15$ held-out targets.
At $T=0.50$, training recovery is $5/15$ and held-out recovery is $4/15$.
Held-out similarity is also highest at $T=0$.

\begin{table}[t]
\centering
\small
\setlength{\tabcolsep}{4pt}
\begin{tabular}{lcccccc}
\toprule
& \multicolumn{3}{c}{\textbf{Train}} & \multicolumn{3}{c}{\textbf{Held out}} \\
\cmidrule(lr){2-4} \cmidrule(lr){5-7}
\textbf{Temperature} & Exact$\uparrow$ & Sim.$\uparrow$ & Rep.$\downarrow$ & Exact$\uparrow$ & Sim.$\uparrow$ & Rep.$\downarrow$ \\
\midrule
$T=0$ & $6/15$ & $0.863$ & $0.505$ & $8/15$ & $0.892$ & $0.452$ \\
$T=0.25$ & $3/15$ & $0.830$ & $0.455$ & $5/15$ & $0.860$ & $0.412$ \\
$T=0.50$ & $5/15$ & $0.861$ & $0.333$ & $4/15$ & $0.839$ & $0.345$ \\
$T=0.75$ & $1/15$ & $0.807$ & $0.234$ & $3/15$ & $0.839$ & $0.231$ \\
\bottomrule
\end{tabular}
\caption{\textbf{Semantle search under decoding temperature.}
Exact matches, mean best semantic similarity, and repetition rate, using the protocol of \cref{tab:search-baselines} on the same checkpoint and the mean-box domain of \cref{eq:method-box}.
Repetition is defined as in \cref{tab:search-domain}.
$T=0$ is greedy decoding as in the primary runs.}
\label{tab:search-temperature}
\end{table}

\subsection{Effect of GP kernel on search}
\label{app:experiments-kernel}

The primary Semantle runs use an ARD Mat\'ern-$2.5$ covariance in the Gaussian-process surrogate.
We also compare Mat\'ern-$1.5$, Mat\'ern-$0.5$, and the squared-exponential (RBF) kernel on the same checkpoint, warm starts, budget, acquisition, and mean-box domain, changing only the kernel family.

\Cref{tab:search-kernel} reports the results.
Mat\'ern-$2.5$ recovers the most targets on both splits ($6/15$ training, $8/15$ held out).
The rougher Mat\'ern kernels and the smoother RBF kernel each reduce exact-match recovery, with the largest drop on held-out targets ($4/15$ for Mat\'ern-$0.5$ and RBF).
Mean best similarity follows the same ordering.
Under Mat\'ern-$0.5$ the repetition rate falls to $0.321$ on training targets and $0.365$ on held-out targets, and held-out exact-match recovery is lower than under Mat\'ern-$2.5$.
On this protocol, Mat\'ern-$2.5$ remains the strongest default among the kernels we tested.

\begin{table}[t]
\centering
\small
\setlength{\tabcolsep}{4pt}
\begin{tabular}{lcccccc}
\toprule
& \multicolumn{3}{c}{\textbf{Train}} & \multicolumn{3}{c}{\textbf{Held out}} \\
\cmidrule(lr){2-4} \cmidrule(lr){5-7}
\textbf{Kernel} & Exact$\uparrow$ & Sim.$\uparrow$ & Rep.$\downarrow$ & Exact$\uparrow$ & Sim.$\uparrow$ & Rep.$\downarrow$ \\
\midrule
Mat\'ern-$2.5$ & $6/15$ & $0.863$ & $0.505$ & $8/15$ & $0.892$ & $0.452$ \\
Mat\'ern-$1.5$ & $3/15$ & $0.834$ & $0.533$ & $6/15$ & $0.865$ & $0.454$ \\
Mat\'ern-$0.5$ & $5/15$ & $0.857$ & $0.321$ & $4/15$ & $0.857$ & $0.365$ \\
RBF & $4/15$ & $0.837$ & $0.493$ & $4/15$ & $0.832$ & $0.498$ \\
\bottomrule
\end{tabular}
\caption{\textbf{Semantle search under GP kernel.}
Exact matches, mean best semantic similarity, and repetition rate, using the protocol of \cref{tab:search-baselines} on the same checkpoint and the mean-box domain of \cref{eq:method-box}.
Repetition is defined as in \cref{tab:search-domain}.
All kernels use ARD lengthscales.
Mat\'ern-$2.5$ is the primary setting.}
\label{tab:search-kernel}
\end{table}

\subsection{Effect of acquisition function on search}
\label{app:experiments-acquisition}

The primary Semantle runs select the next code by log expected improvement.
We also compare upper confidence bound, with coefficient $\beta = 0.2$, and Thompson sampling on the same checkpoint, warm starts, budget, ARD Mat\'ern-$2.5$ kernel, and mean-box domain.

\Cref{tab:search-acquisition} reports the results.
Log expected improvement recovers the most targets on both splits ($6/15$ training, $8/15$ held out).
Upper confidence bound recovers $4/15$ and $6/15$.
Thompson sampling recovers $2/15$ and $3/15$, with repetition rates $0.201$ and $0.205$.
On this protocol, log expected improvement remains the strongest default among the acquisition functions we tested.

\begin{table}[t]
\centering
\small
\setlength{\tabcolsep}{4pt}
\begin{tabular}{lcccccc}
\toprule
& \multicolumn{3}{c}{\textbf{Train}} & \multicolumn{3}{c}{\textbf{Held out}} \\
\cmidrule(lr){2-4} \cmidrule(lr){5-7}
\textbf{Acquisition} & Exact$\uparrow$ & Sim.$\uparrow$ & Rep.$\downarrow$ & Exact$\uparrow$ & Sim.$\uparrow$ & Rep.$\downarrow$ \\
\midrule
LogEI & $6/15$ & $0.863$ & $0.505$ & $8/15$ & $0.892$ & $0.452$ \\
UCB & $4/15$ & $0.845$ & $0.461$ & $6/15$ & $0.865$ & $0.442$ \\
Thompson sampling & $2/15$ & $0.813$ & $0.201$ & $3/15$ & $0.826$ & $0.205$ \\
\bottomrule
\end{tabular}
\caption{\textbf{Semantle search under acquisition function.}
Exact matches, mean best semantic similarity, and repetition rate, using the protocol of \cref{tab:search-baselines} on the same checkpoint and the mean-box domain of \cref{eq:method-box}.
Repetition is defined as in \cref{tab:search-domain}.
LogEI is the primary setting.
UCB uses coefficient $\beta = 0.2$.}
\label{tab:search-acquisition}
\end{table}

\subsection{Training ablation numbers}
\label{app:experiments-ablations-table}

\Cref{fig:search-cosine} shows the anytime curves for the runs summarized on the right of \cref{fig:search-n}.
\Cref{tab:search-breadth} gives the train and held-out split.
Its interpolation column is the median distance defined in \cref{app:experiments-interp}: temperature-$1$ decodes along $40$ training-target pairs at $t\in\{0.4,0.5,0.6\}$, compared with an endpoint-only reference at $0.365$.
Similarity in that panel is the unweighted mean of the two split means below, and the exact-match label is their sum.

\begin{figure}[t]
\centering
\includegraphics[width=\linewidth]{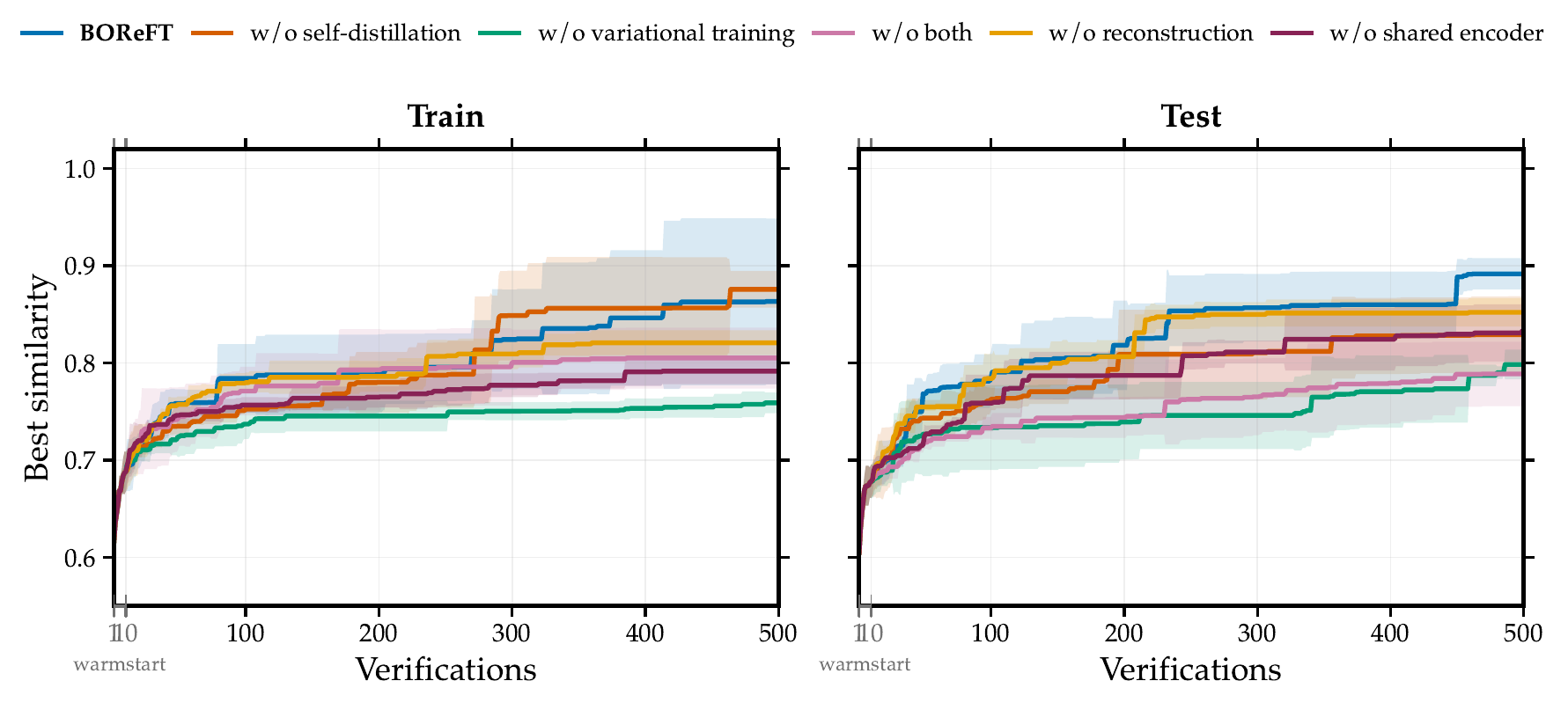}
\caption{\textbf{Search under training ablations.}
Best-so-far semantic similarity (mean $\pm$ 1 SD across 3 seeds) on five training and five held-out Semantle targets. The first ten objective evaluations are shared warm starts.}
\label{fig:search-cosine}
\end{figure}

\begin{table}[t]
\centering
\small
\setlength{\tabcolsep}{4pt}
\begin{tabular}{lccccc}
\toprule
& \multicolumn{2}{c}{\textbf{Train}} & \multicolumn{2}{c}{\textbf{Held out}} & \\
\cmidrule(lr){2-3} \cmidrule(lr){4-5}
\textbf{Method} & Exact$\uparrow$ & Sim.$\uparrow$ & Exact$\uparrow$ & Sim.$\uparrow$ & $\Delta_\text{interp}\downarrow$ \\
\midrule
BOReFT & $\mathbf{6/15}$ & $0.863$ & $\mathbf{8/15}$ & $\mathbf{0.892}$ & $\mathbf{0.338}$ \\
\quad \textit{w/o self-distillation} & $\mathbf{6/15}$ & $\mathbf{0.876}$ & $3/15$ & $0.829$ & $0.355$ \\
\quad \textit{w/o reconstruction} & $2/15$ & $0.821$ & $5/15$ & $0.852$ & $0.396$ \\
\quad \textit{w/o shared encoder} & $1/15$ & $0.792$ & $4/15$ & $0.833$ & $0.382$ \\
\quad \textit{w/o variational training} & $0/15$ & $0.759$ & $3/15$ & $0.799$ & $0.356$ \\
\bottomrule
\end{tabular}
\caption{\textbf{Search and interpolation in ablated spaces.}
Exact matches and mean best semantic similarity on Semantle, split by whether the hidden target was included in representation training (15 runs per split: 5 targets $\times$ 3 seeds). $\Delta_\text{interp}$ is the median distance from the average embedding of temperature-$1$ decodes to the semantic interpolant $x_\alpha$, over 40 target pairs at $t\in\{0.4,0.5,0.6\}$. An endpoint-only reference that always emits the nearer training word has $\Delta_\text{interp}=0.365$.}
\label{tab:search-breadth}
\end{table}

\subsection{Effect of intervention rank}
\label{app:experiments-rank}

The primary Semantle runs use intervention rank $r=64$.
We also vary the rank at fixed training-set size $N=3072$, keeping the rest of the protocol in \cref{sec:experiments-main} unchanged.
\Cref{fig:search-rank} reports mean best semantic similarity and exact-match rate, pooled over the same ten targets and three seeds as \cref{fig:search-n}.
Rank $64$ gives the strongest exact-match result ($14/30$).
Performance is not monotone in rank, with a pronounced dip at rank $32$.

\begin{figure}[t]
\centering
\includegraphics[width=0.72\linewidth]{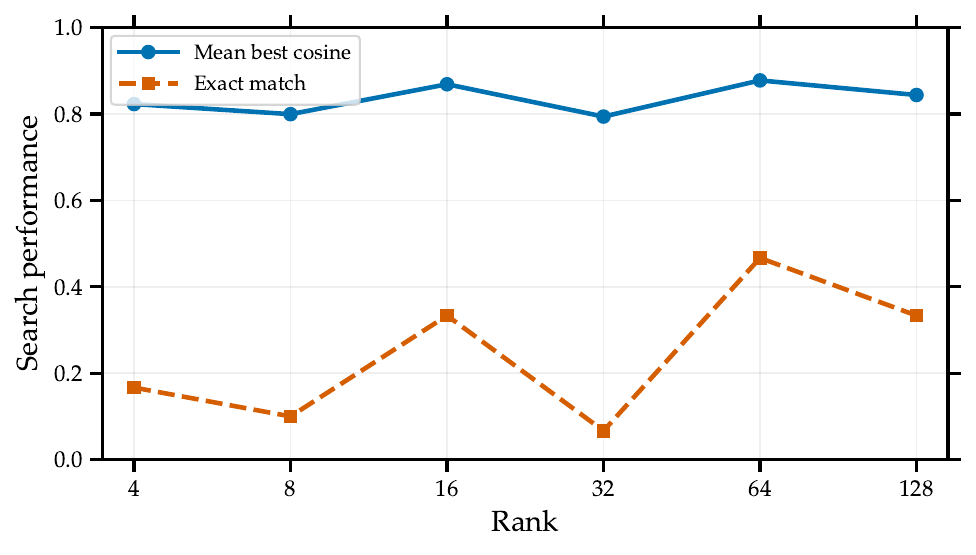}
\caption{\textbf{Search vs.\ intervention rank.}
Mean best semantic similarity (solid) and exact-match rate (dashed) at $N=3072$, pooled over 10 targets $\times$ 3 seeds.}
\label{fig:search-rank}
\end{figure}

\newpage
\section{Additional related work}
\label{app:related}

This section records comparisons that sit next to \cref{sec:related} but are not needed to follow it.

\paragraph{Test-time search and adaptation.}
OPRO conditions generation on scored candidates \citep{yang2024opro}, Tree of Thoughts searches intermediate states \citep{yao2023tree}, and AutoDiscovery searches a tree of hypotheses \citep{agarwal2025autodiscovery}.
Test-time training updates parameters at inference \citep{sun2020testtime,hardt2024test,akyurek2025surprising}.
MiGrATe uses mixed-policy GRPO \citep{phan2025migrate}, and SDPO distills from feedback on hard problems \citep{hubotter2026reinforcement}.
BOReFT learns its search domain before inference and then holds the generator and the domain fixed.

\paragraph{Verbalizable representations.}
Beyond sparse autoencoders, the Jacobian lens and its associated J-space identify representations that can be verbalized as text \citep{gurnee2026verbalizable}.
These results describe structure that is already present in a model's activations.
BOReFT does not try to recover that structure.
It learns a separate set of interventions from target solutions.

\paragraph{Monitoring and manifold steering.}
Representation Engineering uses population-level representations for monitoring as well as for manipulation \citep{zou2023repe}.
Manifold-steering methods go beyond linear directions by fitting activation manifolds and intervening along paths that respect their learned geometry \citep{wurgaft2026manifold}.
Both lines of work aim to describe or follow geometry that the model already has.
The BOReFT intervention is instead a learned search domain, and its coordinates are chosen by an external optimizer.

\end{document}